\documentclass[12pt,draftcls,onecolumn]{IEEEtran}

\usepackage[cmex10]{amsmath}
\usepackage{lmodern,bm}
\usepackage{subfigure}
\usepackage[usenames,dvipsnames]{color}
\usepackage{algorithm}
\usepackage[noend]{algpseudocode}
\usepackage[pdftex]{graphicx}
\usepackage{epstopdf}
\usepackage{multirow}
\usepackage{url}

\usepackage{color, colortbl}
\definecolor{Red}{rgb}{1,0,0}

\usepackage{amsthm}
\theoremstyle{plain}
\newtheorem{theorem}{Theorem}[section]
\newtheorem{lemma}{Lemma}
\newtheorem{corollary}[theorem]{Corollary}
\newtheorem{assumption}{Assumption}

\theoremstyle{definition}

\theoremstyle{remark}
\newtheorem{nota}{Remark}

\usepackage{amssymb}

\usepackage{tikz}
\usetikzlibrary{calc}
\usetikzlibrary{shapes,arrows,fit}
\include{tikz-3dplot} 
\usepackage{ifthen}
\usepackage{etoolbox}
\usepackage{todonotes}
\usepackage{pgfplots}

\usepackage{verbatim}
\usetikzlibrary{patterns}

\usepackage{booktabs}

\newcommand{\minunder}[2]{\underset{#1}{\operatorname{min}} ~ {#2}}
\newcommand{\limunder}[2]{\underset{#1}{\operatorname{lim}} ~ {#2}}
\definecolor{colorTabRow}{gray}{0.8}

\date{}

\begin{document}
%

\title{Interactive Image Segmentation From A Feedback Control Perspective}
%
%
%

\author{Liangjia Zhu,
	   Peter Karasev,
       Ivan Kolesov,
       Romeil Sandhu,
       and Allen Tannenbaum,~\IEEEmembership{Fellow,~IEEE}
\thanks{L. Zhu, I. Kolesov, Romeil Sandhu, and A. Tannenbaum are with the Department
of Computer Science and Applied Mathematics/Statistics, Stony Brook University: (emails: \{liangjia.zhu, ivan.kolesov, romeil.sandhu, allen.tannenbaum\}@stonybrook.edu)

P. Karasev is now with Exception Orthopaedic Solutions, LLC.}
}

\maketitle

\begin{abstract}
Image segmentation is a fundamental problem in computational vision and medical imaging. Designing a generic, automated method that works for various objects and imaging modalities is a formidable task. Instead of proposing a new specific segmentation algorithm, we present a general design principle on how to integrate user interactions from the perspective of feedback control theory. Impulsive control and Lyapunov stability analysis are employed to design and analyze an interactive segmentation system. Then stabilization conditions are derived to guide algorithm design. Finally, the effectiveness and robustness of proposed method are demonstrated.

\end{abstract}


\IEEEpeerreviewmaketitle

\section{Introduction}

\IEEEPARstart{T}{he} problem of image segmentation has been an active research field over the past several decades and
remains as a very challenging task.  Although user (human) knowledge can recognize and partition an image into necessary regions,
current automated algorithms fail to capture such boundaries on a consistent basis over wide ranging image modalities, i.e., there exists no ``universal'' segmentation algorithm. This said, the issue of effectively integrating user prior knowledge into a segmentation design is a
driving principle behind existing state-of-the-art methods. These application-driven methodologies utilized prior models to aid in the segmentation process \cite{MedSeg_Shape_Review}. However, these methods remain automated and suffer from the same tacit issues for which they were designed to overcome. As a result, users directly participate in the segmentation loop in various image segmentation systems \cite{SiemensWS, ITKSnap}. Given such input, the research community has classified this area of research as \textit{interactive segmentation} \cite{IntelScissors, GraphCut_ND_IJCV,GrabCut, GrowCut, Interactive_RW_Cardiac}. Typically, one starts by initializing a method and then iteratively modifies the intermediate results until the algorithm obtains a satisfactory result. This process loop can be alternatively viewed as a feedback
control process \cite{Feedback_Textbook}, whereby the user's editing action is a controller, the visualization/monitoring is an observer, and the changes of the intermediate results drive the system dynamics. Once the segmentation is complete, the system converges to an expected result.
In this manner, the user automatically fills the ``information'' gap between an imperfect model to that of a desired segmentation. However, much of the work along this line of research does not explicitly model the user's role from the systems and control perspective. That is, user inputs are used \textit{passively} without consideration to its contribution to the stability or convergence of the overall system. To the best of our knowledge, there are  very few attempts to model segmentation from the perspective of feedback control.

Feedback principles have been used in the literature either based on empirical rules, such as boundary consistency \cite{AdaptiveSeg_SMC}, connectivity of foreground pixels \cite{ControlPatternRecog_IFAC}, or by modeling user contribution as a weighted term in an objective function \cite{Interactive_ProbLevelSet_SPIE, InteractiveLS_ISBI} to \textit{close} the segmentation process. The important property of stability for a control system  was not touched upon in these works. In our previous work \cite{KSlice_TMI}, we formulated an interactive image segmentation methodology as a form of feedback control of a given PDE system and then derived its stability conditions. The method was chosen to be the classical region-based active contours for single object segmentation. One major advantage of this formulation is that although the user does not know a perfect model/method for a segmentation task (which in fact rarely exists), one can start from a classic model and alter the segmentation in a principled manner. We can then quantitatively design and evaluate the performance a control-based segmentation system.

In this paper and with regards to our previous work \cite{KSlice_TMI}, we expand our methodology to generalized cases of evolution-based segmentation methods in addition to supporting multiple-object segmentation. Figure \ref{fig:digram} shows the diagram of the proposed framework. A user
incorporates their prior knowledge to generate corrections as input to the closed-loop system.
The segmentation boundary evolution and explicit estimate of the user's ideal segmentation are
updated within an inner loop.

\begin{figure*}[!htb]
\begin{center}
	\includegraphics[width = 6.5in]{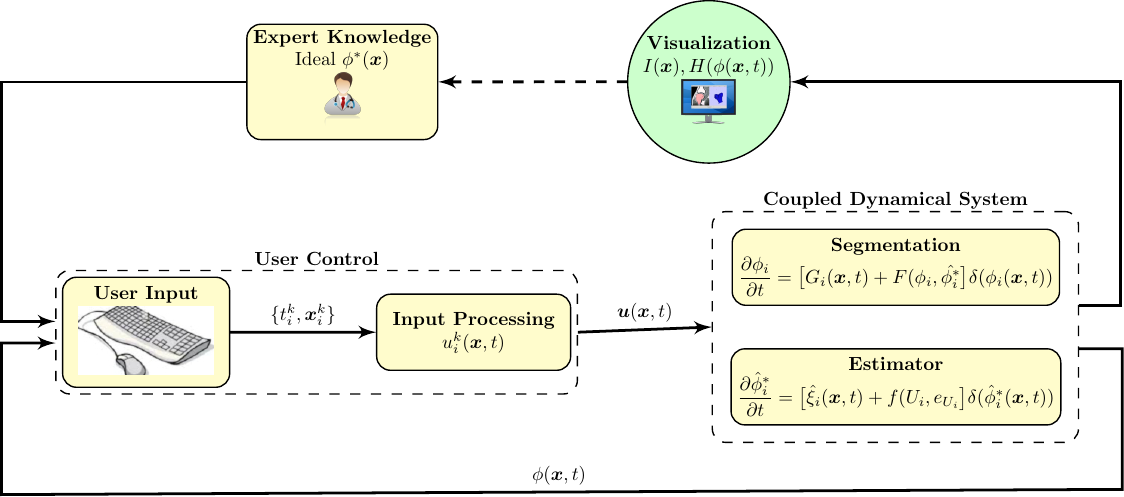}%
\end{center}\vspace{-0.5cm}
\caption{Diagram of the control-based segmentation framework.The feedback compensates for deficiencies in automatic segmentation by utilizing the expert's knowledge.}
\label{fig:digram}
\end{figure*}

The main contributions of this work are:
\begin{itemize}
\item{We present a framework on how to design an interactive segmentation system from a feedback control perspective. Stabilization conditions for an interactive segmentation system are derived yielding a tangible framework for algorithm design and analysis. A key feature is that we are bridging image segmentation with control theory. As such, we are able to leverage new developments to design and analyze more sophisticated image segmentation systems. For example and in this work, impulsive control is adopted to model user input so as to provide a more practical approach than the continuous case derived in \cite{KSlice_TMI}}
\item{We present a relaxed framework with regards to the core algorithm (i.e., the only requirement is the ability to be described as a dynamical system). In short, this can be applied to multi-object segmentation, handles both region- and distance-based metrics, and supports scalar and vector images.}
\item{We present a framework capable of analyzing the behavior of some existing interactive segmentation methods, providing directions for further improvements.}
\end{itemize}
We should note and emphasize that the proposed controlled segmentation framework offers control theoretics to enhance existing methods as opposed to replacing the method itself.

This paper is organized as follows. Section~\ref{sec:Review} reviews automated and interactive segmentation methods as well as basics from feedback control theory. Section~\ref{sec:AutoSeg} gives the formulation of an automated segmentation as a dynamical system for both region- and distance-based metrics. Section~\ref{sec:InterativeSeg} provides the control laws to stabilize the dynamical system for generic cases and shows specific examples on how to realize these control laws. Section~\ref{sec:Results} reports segmentation results followed by conclusion in Section~\ref{sec:Conclusion}.

\section{Literature Review}
\label{sec:Review}
\subsubsection{\textbf{(Semi-)Automatic Image Segmentation}}
~\\
Variational image segmentation or active contour models are one class of algorithms commonly used in automated image segmentation. One underlying assumption employed in these algorithms is that the optimization of an energy functional defined over image features leads to an expected partition of the image \cite{MumfordShah}. Examples of commonly used image features are edges \cite{Geometric_AC_ICCV, GAC_IJCV}, regional statistics \cite{Chan_Vese, CoupledAC_Yezzi}, and their combinations \cite{GeodescRegion_IJCV, Local_Global}. Multi-object segmentation has been addressed either by adding constraints to penalize violating areas \cite{Zhao_MultiMotion,MultiRegionAC_TIP06} or by introducing competing  components for adjacent regions \cite{ACClustering_PAMI06,RSS}. The shortest-distance-based image segmentation methods have been proposed along this line \cite{MinPath_Cohen}. Image classification/clustering can also be modeled by using variational framework \cite{ACClustering_ICIP04}. Comprehensive introduction and review of variational segmentation methods can be found in \cite{GeometricPDE_Book, ActiveContour_Review} and references therein.

Alternative ways of representation, such as based on graph or clustering, have been widely used in image segmentation as well. See \cite{GraphSeg_Review_PR} and \cite{DataClustering_Review}  and references therein for comprehensive reviews.

\subsubsection{\textbf{Interactive Image Segmentation}}
~\\
Automatic segmentation is attractive, but user intervention is inevitable in some critical tasks, especially for medical images \cite{InteractiveMedSeg_Review}. The limitations of automatic or semi-automatic segmentation methods are well known and have been extended to integrate user interactions. Depending on the focus of these methods, the interactive segmentation methods are roughly summarized as follows:

\paragraph{User interface}
Some classic automatic segmentation methods have been combined with powerful visualization and user editing system to achieve interactive segmentation. This is popular in some medical image segmentation systems \cite{SiemensWS, ITKSnap,Interactive_LiveSurface}, where classic methods like region growing, active contours, and graph cuts are adopted as core techniques. Common features  shared by these methods are efficient user interaction, easiness of manipulation, especially for 3D volumes.

\paragraph{ Segmentation accuracy}
Much of the  research focus has placed emphasis on improving the segmentation accuracy per user input. Typically, image segmentation is associated with the optimization of an objective function. Thus, advances have been made either by improving classic objective functions such as active contours \cite{InteractiveLS_ISBI}, graph cuts \cite{Interactive_GeodesicGraphCut}, and geodesic distance \cite{GeodesicSeg_IJCV},  or by proposing new formulations to interpret user input such as conditional random field \cite{Interactive_GeoS} and random walker \cite{Interactive_RW_Cardiac}. One of notable advances in this direction is the reformulation of some variational algorithms, which are subject to local minima,  into convex optimization framework  with global optimal solutions \cite{Interactive_ConvexActiveContour, TVSeg_Cremers,ConvexOpt_PAMI_Cremers}.

\paragraph{User input modeling and accumulation} How to understand the meaning of user input to better equip a segmentation algorithm has recently received attention in the literature. User input has been modeled as non-Euclidean kernels \cite{RBF_Input_Yezzi}, combined with shape constraint \cite{Interactive_GeoStarConnec_CVPR}, and  categorized \cite{UserFriendlyInteraction_TIP} to reduce user efforts in interaction. Since user interaction is not an isolated and passive action, learning-based methods have been proposed to model the temporal/historic knowledge of user input to increase efficiency  \cite{LearningBasedInteractSeg_ACCV, ActiveLearning_Patch_IPMI}.

\paragraph{System design} With regards to system design, recent work has illuminated the \textit{causality} and overall performance of a system. User input is an integral part of the system, which is combined with the current and previous states to determine the next segmentation. Most of these algorithms are rooted either in variational  \cite{KSlice_TMI, Interactive_ImplictFunc, InteractLevelSet_ISBI} or graph-based segmentations \cite{IntelScissors, GraphCut_ND_IJCV, GrabCut, GrowCut, Interactive_DIFT_TMI}.

In addition to these four categories of research, GPU computing has also been increasingly utilized to develop real-time interactive segmentation systems \cite{Interactive_ProbLevelSet_SPIE, TVSeg_Cremers,ConvexOpt_PAMI_Cremers}.

\subsubsection{\textbf{Feedback Control Theory}}
~\\
The principle of feedback is most often used in a control system. The basic idea is to use the difference between the signal to be controlled and a desired reference signal to determine system actions. \textit{We should note that feedback control ought not to be confused with simply taking feedback in a controlled system - there is subtle but distinct difference to unambiguously indicate when the performance is guaranteed} \cite{Feedback_Textbook,JoyOfFeedback}. The second requirement can be considered as studying the stability of a system around a given state or an equilibrium point, referred to as  \textit{stabilization} of a system.

A common technique used in system stabilization is by analyzing the Lyapunov function of a system \cite{NonlinearSystem_Book_Khalil, NonlinearSystem_Book_Marquez, Lyapunov_Discontinous_Clarke}. The basic idea is to check whether the Laypunov function is dissipative along all possible trajectories of a dynamical system around the equilibrium point, rather than solving the system equation directly. It has been applied to stabilize systems driven by single \cite{Global_Stable_Burgers} as well as coupled \cite{Lyapunov_CoupledPDEs} PDEs. Lyapunov analysis has been extended to derive the stability condition for time-delayed system \cite{Lyapunov_Delay}, discrete-time system \cite{Lyapunov_Discrete_SCL}, networks \cite{Lyapunov_Networks}, and to synchronize two chaotic systems with impulsive signals \cite{Impulsive_Chaos_CSF}.

Impulsive control is a rather recent control development derived from the theory of impulsive differential equations \cite{ImpulsiveControl_Book}. It has been widely used in synchronizing dynamical systems due to its effectiveness and efficiency as it requires only small control efforts to stabilize dynamical systems \cite{Impulsive_Chaos_PRL}.

The existence of feedback has long been observed and utilized in modeling the human vision system \cite{ART_Grossberg}. However, only very few attempts have been made to employ the feedback principle to design an image processing system. Early research add feedback in image segmentation either based on empirical rules \cite{AdaptiveSeg_SMC,ControlPatternRecog_IFAC} or by adding user contribution as a weighted term \cite{Interactive_ProbLevelSet_SPIE, InteractiveLS_ISBI}, without addressing the issue of stability for these closed-loop systems. In our previous work \cite{KSlice_TMI}, interactive image segmentation is formulated as controlling region-based active contours, where control law is derived rigorously by using the Laypunov stability theorem.

\begin{figure*}[!htb]
\centering
\subfigure{\includegraphics[height = 0.9in]{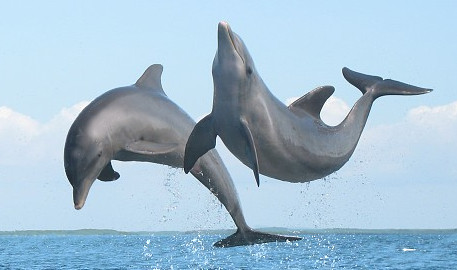}}
\subfigure{\includegraphics[height = 0.9in]{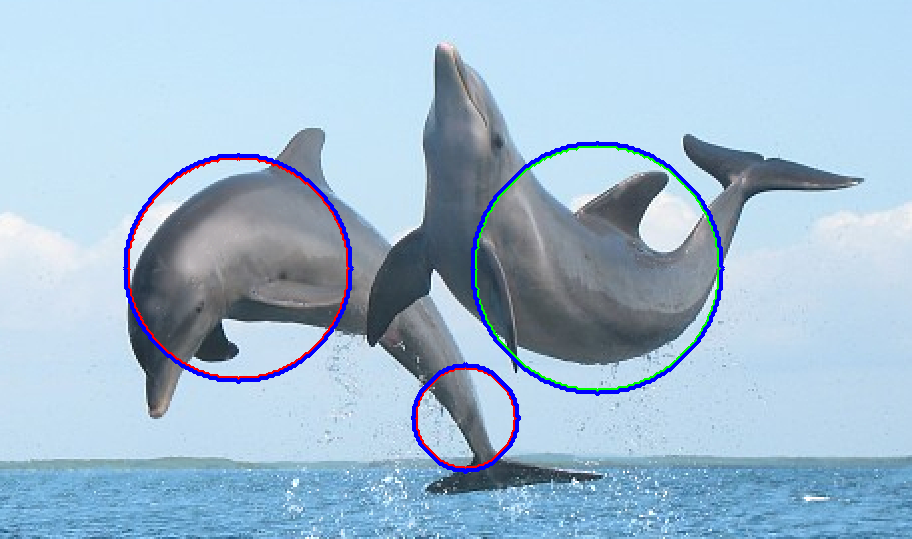}}
\subfigure{\includegraphics[height = 0.9in]{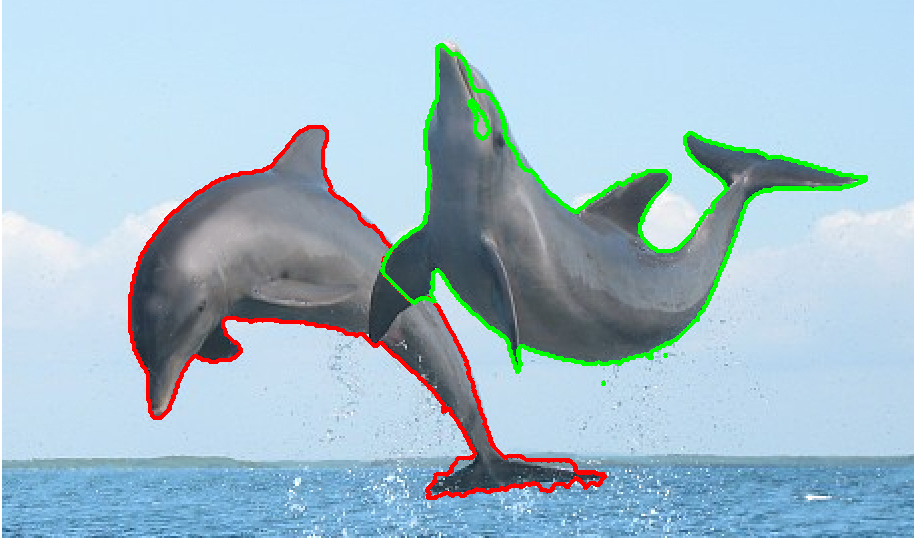}}\\  \vspace*{-0.5em}
\subfigure{\includegraphics[height = 0.9in]{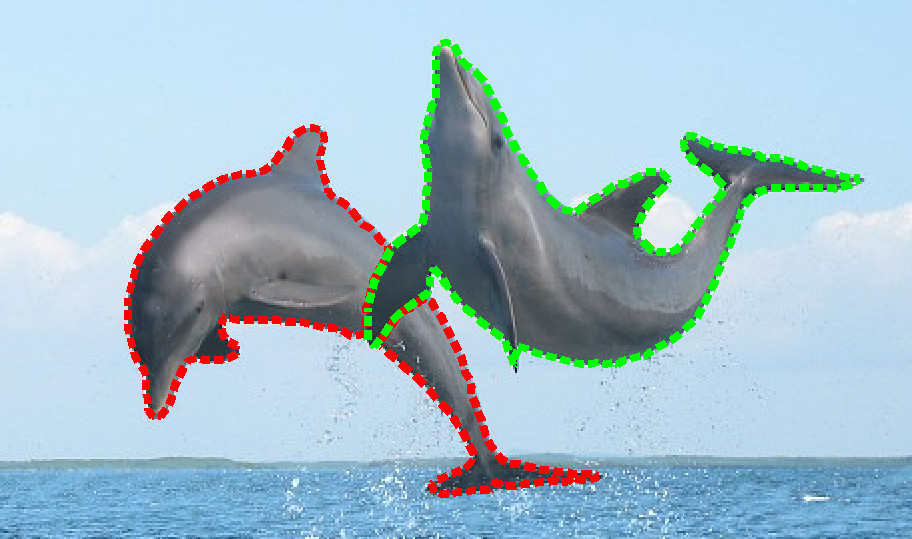}}
\subfigure{\includegraphics[height = 0.9in]{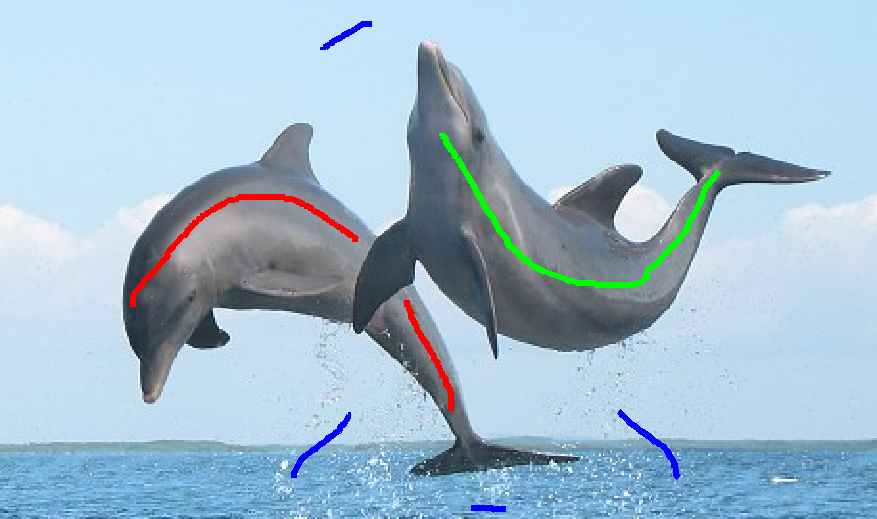}}
\subfigure{\includegraphics[height = 0.9in]{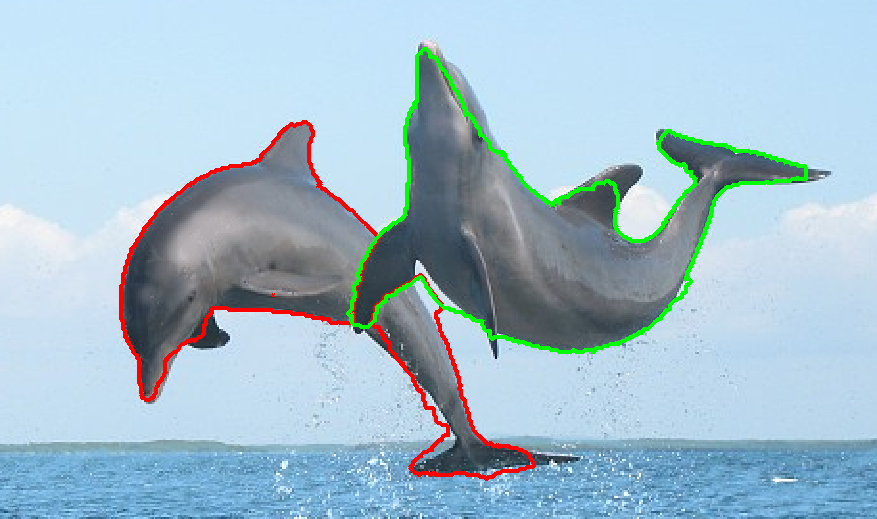}} \vspace{-0.3cm}
\caption{ Example of automatic segmentation. First column: source image and manual segmentation. Second column: initialization for region-based active contour (top) and distance-based clustering (bottom). Last column: automatic segmentation results.}
\label{fig:AutoSeg}
\end{figure*}

\section{Automatic Segmentation As An Open-Loop System}
\label{sec:AutoSeg}
A large class of segmentation algorithms can be considered as evolutionary processes. Starting from some given regions, these algorithms evolve these regions (or their boundaries) based on certain quantifiable criteria. Examples include classical active contour models, and distance-based segmentation. Typically, the evolutionary process can be described by a dynamical system, driven by the optimization of certain energy functionals. As an example, the level set formulation of active contours has been extensively employed; see \cite{FastMarching_Book, LevelSet_Fedkiw} and the references therein. We note that level sets were proposed in \cite{LevelSet_1980} and subsequently developed and applied in many works.

An image can be represented by function $I:\Omega \rightarrow \mathbb{R}^n$ defined on $\Omega \in \mathbb{R}^m$, where $m\geq 2$ and $n\geq 1$. Suppose at any given time $t\in \mathbb{R}^+$, the image can be partitioned into \textit{N} regions $\Omega_i(\bm{x},t)$ that fully cover the image without overlaps.

Each region $\Omega_i(\bm{x}, t)$ is associated with a level set function $\phi_i(\bm{x}, t)$ and is evolving from an initial state $\phi_i(\bm{x},0) = \phi_i^0(\bm{x})$. The regularized Heaviside function
\begin{equation}
\label{equ:Heaviside}
H(\phi) = \begin{cases}
1, &\mbox{if } \phi > \epsilon\\
0, &\mbox{if } \phi < -\epsilon\\
\dfrac{1}{2}\left(1+\dfrac{\phi}{\epsilon}+\dfrac{1}{\pi}\sin\big(\dfrac{\pi\phi}{\epsilon}\big)\right), &\mbox{otherwise}
\end{cases}
\end{equation}
is used to indicate the exterior and interior regions and its derivative is denoted by the regularized delta function $\delta(\phi)$.
In this paper, we take $\phi_i \geq 0$ to indicate the inside of a region $\Omega_i$.

The dynamical system describing the evolution is defined as follows,
\begin{equation}
\label{equ:Dyanmics}
\begin{split}
\dfrac{\partial{\phi_i}}{\partial{t}}&= G_i(\bm{x},t)\delta(\phi_i)\\
\phi_i(\bm{x},0) &= \phi_i^0(\bm{x})
\end{split},
\end{equation}
where $G_i:\mathbb{R}^m\times\mathbb{R}^+\rightarrow \mathbb{R}$ is the intrinsic dynamics that describes evolution of region $\Omega_i(\bm{x},t)$. Without loss of generality, each $G_i(\bm{x},t)$ is decomposed into two \textit{competing} components as
\begin{equation}
\label{equ:G}
G_i(\bm{x},t) = - \left[ g_i(\bm{x},t) - g_i^c(\bm{x},t) \right]
\end{equation}
where $g_i:\mathbb{R}^m\times\mathbb{R}^+\rightarrow \mathbb{R}$ represents the contribution
 from  $\Omega_i(\bm{x},t)$ and $g_i^c:\mathbb{R}^m\times\mathbb{R}^+\rightarrow \mathbb{R}$ is that from all other $\Omega_j(\bm{x},t)$ for $j\neq i$. A negative sign is used in front of equation \eqref{equ:G} due to the definition of $H$ for interior region. This way of decomposition has been used for modeling multiple active contours in different region-based algorithms \cite{ACClustering_PAMI06,ACClustering_ICIP04, RSS}. On the other hand, segmentation algorithms that are based on clustering pixels according to their ``distances'' to given seed points naturally fit into this formulation, since image-based distance to given points can be implemented using the level set formulation \cite{MinPath_Cohen}.  Therefore, the presented framework works for: 1) \textit{region-based active contour models} and 2) \textit{distance-based clustering}.

\subsection{Region-based Active Contour Models}
\label{sec:region_auto}
The function $g_i(\bm{x},t)$ may be defined as the statistics of $ I(\bm{x})$ inside the region $\Omega_i(\bm{x},t)$. A simple example is the first order statistics given as
\begin{equation}
\label{equ:GCV}
g_i(\bm{x},t) := \left[ I(\bm{x}) - \mu_i(t) \right]^2 \delta(\phi_i)
\end{equation}
where $I(\bm{x}) \in \mathbb{R}$ and $\mu_i(t)$ is the average value of $I(\bm{x})$ inside the region $\Omega_i$ at time $t$ \cite{Chan_Vese, CoupledAC_Yezzi}; and
\begin{equation}
g_i^c(\bm{x}_k,t) := \minunder{j\neq i}{g_j(\bm{x}_k,t)} \mbox{\quad for each }\bm{x}_k \in \Omega_i \ .
\end{equation}
Other region-based energies \cite{Local_Global, RSS, Bhattacharyya_AC} may be employed in a similar way.

\subsection{Distance-based Clustering}
\label{sec:dist_auto}
Let $\Omega_{\bm{x}_i}$ be a set of points inside a region $\Omega_i$, which is referred to as ``seed points" herein after. The distance between a point $\bm{x} \in \Omega$ and the seed region is defined as
\begin{equation}
\label{equ:Dist2Seed}
d(\bm{x}, \Omega_{\bm{x}_i}) = \min_{\bm{y} \in \Omega_{\bm{x}_i}} \min_{C \in \theta(\bm{x},\bm{y})} \int_0^1 {g_{\gamma}(C(p))\Vert C'(p) \Vert dp},
\end{equation}
where $\theta(\bm{x},\bm{y})$ is the family of all paths connecting points $\bm{x}$ and $\bm{y}$, and $p \in [0, 1]$ is the parametrization of a specific path $C:[0,1] \rightarrow \mathbb{R}^m$ weighted by an image-based function $g_\gamma :\mathbb{R}^m \rightarrow \mathbb{R}^+$. The distance $d(\bm{x}, \Omega_{\bm{x}_i})$ may be computed using the level set formulation as well by interpreting it as a front propagation problem with an image-dependent distance measure $1/g_\gamma$, where $g_{\gamma}(I) =1+\Vert \nabla I \Vert_2^2$ \cite{MinPath_Cohen}. After computing the distance from a point to each seed region, the point is assigned to the closest region.

With a slight abuse of notation, let $\phi_i^{-1}(\bm{x},t)$, $\phi_{i^c}^{-1}(\bm{x},t)$, and $\phi_{\text{min}}^{-1}(\bm{x},t)$  be the distance between point $\bm{x}$ and region $\Omega_i$, the shortest distance between the point $\bm{x}$ and any regions other than $\Omega_i(\bm{x},t)$, and the shortest distance between the point and all regions, respectively. An example of natural dynamics acting on $\phi_i$ is defined as
\begin{equation}
\label{equ:GDistTerm}
g_i(\bm{x},t) := \left\{
         		\begin{array}{rlll}
             		 &g_{\gamma}(I) 			  &\mbox{if $\phi_i^{-1}(\bm{x},t) \neq \phi^{-1}_{\text{min}}(\bm{x},t) $ }  & \\
             		 &0  						  &\mbox{otherwise }   &
         		\end{array}
      		\right.
\end{equation}
and
\begin{equation}
\label{equ:GDistTermC}
g_i^c(\bm{x},t) :=  \left\{
         		\begin{array}{rlll}
             		 &g_{\gamma}(I) 			  &\mbox{if $\phi_{i^c}^{-1}(\bm{x},t) \neq \phi^{-1}_{\text{min}}(\bm{x},t)  $ }  & \\
             		 &0  						  &\mbox{otherwise .}   &
         		\end{array}
      		\right.
\end{equation}
The equations \eqref{equ:GDistTerm} and \eqref{equ:GDistTermC} are used in the evolution of \eqref{equ:Dyanmics} \eqref{equ:G}. This formulation is essentially a clustering processing based on the shortest distance from a point to all regions.

Figure \ref{fig:AutoSeg} shows an example of automatic segmentation by using the models described in this section. With simple user initializations, either circular regions or scribbles, these two classical algorithms capture the majority of the objects, while missing some details. In particular, the region-based active contour has both evolution leakage and unreached regions, while the distance-base clustering fails to mark correct boundaries where seed regions were not properly specified. Though sophisticated methods \cite{Interactive_ConvexActiveContour, GeodesicSeg_IJCV, ConvexOpt_PAMI_Cremers} may achieve better segmentation results with the same initializations, our main focus is how to improve these classical methods by integrating user inputs from the perspective of feedback control.

\section{Interactive Segmentation As Feedback Control}
\label{sec:InterativeSeg}
Suppose the user has an ideal segmentation of the image domain into regions in mind: $\{\phi^*_i(\bm{x})\} \mbox{ for } i=1,\cdots,N$ . Then, the goal is to design a feedback control system
\begin{equation}
\label{equ:dynamics}
\begin{split}
\dfrac{\partial{\phi_i}}{\partial{t}} &= \big[G_i\left(\bm{x},t\right) + F(\phi_i, \phi^*_i) \big]\delta(\phi_i) \\
\phi_i(\bm{x},0) &= \phi_i^{0}(\bm{x})
\end{split},
\end{equation}
such that $\limunder{t\rightarrow \infty}{ \phi_i(\bm{x},t) \rightarrow \phi_i^{*}(\bm{x})}$ for $i=1,\cdots,N$, where
$F(\phi_i,\phi^*_i)$ is the control law to be defined below.

\subsection{Existence of A Regulatory Control}
\label{sec:ExistControl}
Define the pointwise and total labeling error as $\xi_i(\bm{x},t)$ and $V(\xi, t) $, respectively,
\begin{equation}
\label{equ:xi}
\xi_i(\bm{x},t) := H(\phi_i)-H(\phi^*_i)
\end{equation}
\begin{equation}
\label{equ:Lyapunov}
 V(t) = \dfrac{1}{2}\sum_{i=1}^{N}\int_{\Omega}{\xi^2_i(\bm{x},t)d\bm{x}} \ .
\end{equation}
If $\phi^*_i$s are given and $V(\xi,t) \in C^1$, the determination of $F(\cdot,\cdot)$ is straightforward by  applying the Lyapunov's direct method for stabilization. The control signal $F$ uses the bounds of the image-dependent term $G_i(\bm{x},t)$,
\begin{equation}
\label{equ:GM}
g_M(\bm{x}) := \sup_{\forall t \in \mathbb{R}^{+}, i=1,\cdots,N} \{ |G_i(\bm{x},t)|\}
\end{equation}

\begin{theorem}
\label{theorem:continuous}
The control law
\begin{equation}
\label{equ:controllaw}
F(\phi_i,\phi^*_i) = \alpha^2_i(\bm{x},t)\xi_i(\bm{x},t),
\end{equation}
where $\alpha^2_i(\bm{x},t) \geq g_M(\bm{x})$, asymptotically stabilizes the system \eqref{equ:dynamics} from $\{\phi_i(\bm{x},t)\}$ to $\{\phi^*_i(\bm{x})\},i=1,\cdots,N$, for $\epsilon$ defined in equation~\eqref{equ:Heaviside} sufficiently small.

Furthermore, the control law exponentially stabilizes the system with a convergence rate of $e^{-\nu t}$ when
$\xi$ is large in the sense that
\begin{equation}
\label{equ:scaleG}
\rho\int_{\Omega}{\delta^2(\phi_i)\xi_i^2(\bm{x},t)d\bm{x}}\leq \int_{\Omega}{\xi^2_i(\bm{x},t)d\bm{x}},i=1,\cdots, N
\end{equation}
for given constants $\nu > 0, \rho > 0$.

\end{theorem}
See Appendix-I for details of proof.

\begin{nota}
This theorem gives a sufficient condition for the existence of a control law that stabilizes the dynamical system to a given desired steady state value. Intuitively, the control law defines a localized input required to ``break'' the intrinsic dynamics.
\end{nota}

\subsection{Label Error Estimation}
\label{sec:ErrorEst}
In practice, $\{\phi^*_i\}$ is not given or completely available beforehand. User input based on current segmentation is used to predict/estimate the ideal segmentation on the fly.

\subsubsection{User Input Processing}
\label{sec:UserInput}

User interaction is modeled as a binary decision as to whether a given location is correctly labeled as inside or outside the expected segmentation. Let $L=\{1,\cdots, N\}$ be the set of labels corresponding to regions $\Omega_i(\bm{x},t)$.
User inputs are properly applied to the image to capture segmentation errors. Afterwards, the effect of user input is propagated.

Formally, the function $u_i^k :\Omega\times \mathbb{R}^+ \rightarrow \mathbb{R}$ with
\begin{equation}
\label{equ:uik}
u_i^k(\Omega_{\bm{x}_i^{k}}, t_i^{k}) = q, \quad u_i^k(\Omega^c_{\bm{x}_i^k}, t_i^{k}) = p,
\end{equation}
models the $k$th input applied at time $t_i^k$ for the $i$th label, $\Omega_{\bm{x}_i^k}$ and $\Omega_{\bm{x}_i^k}^c:=\Omega\setminus \Omega_{\bm{x}_i^k}$ are the regions the user scribbled and its complement, respectively. Constants $q,p \in \mathbb{R}$  are the initial values of $u_i^k$ at $\Omega_{\bm{x}_i^k}$ and $\Omega_{\bm{x}_i^k}^c$, respectively.

\subsubsection{Accumulation of User Input}
~\\
Let $\bm{u}(\bm{x},t) =\{u_1^1(\bm{x},t),\cdots, u_2^1(\bm{x},t), \cdots\}$ be the set of functions representing all user input effect at time $t$. Let the function $\bm{u}_i$ be the accumulated effect from all $u_i^k(\bm{x},t)$  and $\bm{u}_i^c$ be the effect all $u_j^k(\bm{x},t), j\neq i, j=1,\cdots, N$ and $k=1,2,\cdots$. The total effect of user input for label $i$ comes from the overall supports of $\bm{u}_i(\bm{x},t)$ and $\bm{u}_i^c(\bm{x},t)$, defined as $U_i(\bm{x},t) :=J(\bm{u}_i,\bm{u}_i^c)$, where $J(,)$ is a function for evaluating the strength of support from $\bm{u}_i(\bm{x},t)$ and $\bm{u}_i^c(\bm{x},t)$.
\\
\subsubsection{Label-Error Estimation}
~\\
Let $\{\hat{\phi}^*_i\}$ be an estimate of $\{{\phi}^*_i\}$ and define the error terms as
\begin{equation}
\begin{split}
\hat{\xi}_i(\bm{x},t) &:= H(\phi_i)-H(\hat{\phi}^*_i)\\
e_{U_i}(\bm{x},t)&:=H(\hat{\phi}^*_i)-H(U_i).
\end{split}
\end{equation}
The feedback in equation \eqref{equ:dynamics} will use the estimate $\{\hat{\phi}^*_i\}$,
\begin{equation}
\label{equ:dynamicsApprox}
\begin{split}
\dfrac{\partial{\phi_i}}{\partial{t}} &= [G_i\left(\bm{x},t\right) + F(\phi_i, \hat{\phi}^*_i)]\delta(\phi_i) \\
\phi_i(\bm{x},0) &= \phi_i^{0}(\bm{x})
\end{split},
\end{equation}
The estimator is an observer-like system driven by accumulated user input $U_i$ with error term $e_{U_i}$ as
\begin{equation}
\label{equ:dynamicsEstimator}
\begin{split}
\dfrac{\partial{\hat{\phi}^*_i}}{\partial{t}} &= \big[\hat{\xi_i} + f\big(U_i,e_{U_i})\big]\delta(\hat{\phi}^*_i) \\
\hat{\phi}^*_i(\bm{x},0) &= \phi_i^{0}(\bm{x})
\end{split}.
\end{equation}
Equations \eqref{equ:dynamicsApprox} and \eqref{equ:dynamicsEstimator} form a coupled dynamical system. The total labeling error is defined as 
\begin{equation}
\label{equ:UserInput}
 \text{estimator vs. user input} \quad E(t) := \dfrac{1}{2}\sum_{i=1}^{N}\int_{\Omega}{|U_i|e^2_{U_i}d\bm{x}},
\end{equation}
\begin{equation}
\label{equ:Visualization}
 \text{estimator vs. visualization} \quad \hat{V}(t) := \dfrac{1}{2}\sum_{i=1}^{N}\int_{\Omega}{\hat{\xi}^2_i d\bm{x}}.
\end{equation}
\begin{theorem}
\label{th:CoupledSys}
Let $f(U_i,e_{U_i}) = -|U_i|e_{U_i}$ and consequently
\begin{equation}
\dfrac{\partial{\hat{\phi}^*_i}}{\partial{t}} = \big[\hat{\xi_i} -|U_i|e_{U_i}\big]\delta(\hat{\phi}^*_i).
\end{equation}
Assume that user input has stopped ($U_i$ remains constantly) and Theorem \ref{theorem:continuous} is satisfied. Then the sum $V(t) := E(t) + \hat{V}(t)$ has a negative semidefinite derivative:
\begin{equation}
\label{equ:VTCouple}
V'(t) \leq -\sum_{i=1}^{N}\int_{\Omega}{\delta^2(\hat{\phi}^*_i)\big[\hat{\xi}_i - |U_i|e_{U_i}\big]^2}d\bm{x}
\end{equation}
\end{theorem}
See Appendix-II for details of proof.

A large number of evolutionary methods can be augmented to include user interactions using the proposed design principle. Two representative models of user input $U_i$ are discussed in the following section.

\subsection{Examples Of Control-based Segmentation Methods}

\subsubsection{Controlled Region-based Active Contour Models}
~\\
The natural dynamics $G_i(\bm{x},t)$ described in Section \ref{sec:region_auto} is used.
A kernel-based method \cite{RBF_Input_Yezzi} is employed to model the user input for $u_i^k$ as
\begin{equation}
\label{equ:AC_input}
u_i^k(\bm{x},t_i^{k+}) = h_0(d(\bm{x}, \Omega_{\bm{x}_i^{k}})),
\end{equation}
where $d(\bm{x}, \Omega_{\bm{x}_i^{k}})$ is the weighted distance from $\bm{x}$ to $\Omega_{\bm{x}_i^{k}}$ defined in \eqref{equ:Dist2Seed}, and $h_0$ is a decreasing function with respect to this distance. In this paper, we use $h_0=(d_{max}-d)/(d_{max}-d_{min})$, where $d,d_{max}, d_{min}$ are the distance and its corresponding extrema, respectively. The value of $d_{max}$ controls the maximal range the current input affects. In addition, the values of $p,q$ in equation \eqref{equ:uik} are set as $p =1$ and $q=0$. The overall effect for label $L=i$ is defined as
\begin{equation}
\bm{u}_i(\bm{x},t) = \sum_{k} {u_i^k(\bm{x},t)},
\end{equation}

To  propagate and smooth the effect of user input,  a diffusion process is applied to $\bm{u}_i(\bm{x},t)$ as in \cite{KSlice_TMI}, where
\begin{equation}
\begin{split}
\dfrac{\partial \bm{u}_i}{\partial t} &= \bm{u}_i + \nabla\cdot\Big[H\big((\bm{u}_i/g_{\text{M}})^2-1\big)\nabla \bm{u}_i\Big]\\
\bm{u}_i(\bm{x},0)&=0
\end{split}.
\end{equation}

The total user input effect is defined as
\begin{equation}
U_i(\bm{x},t):=\bm{u}_i(\bm{x},t) - \sum_{j \neq i} \bm{u}_j(\bm{x},t)
\end{equation}

\subsubsection{Controlled Distance-based Clustering}
~\\
In this example, we use the simple natural dynamics $G_i(\bm{x},t)$ as described in Section \ref{sec:dist_auto}. More sophisticated schemes such as \cite{GeodesicSeg_IJCV} may be used. The effect of user input is defined to have the same metric as the system equations,
\begin{equation}
\begin{split}
&u_i^k(\bm{x},t) = \min_{\bm{y} \in \Omega_{\bm{x}_i^k}} \min_{C \in \theta(\bm{x},\bm{y})} \int_0^1 {g_{\gamma}(C(l))\Vert C'(l) \Vert dl}\\
&u_i^k(\Omega_{\bm{x}_i^{k}}, t_i^{k}) = q, \quad u_i^k(\Omega^c_{\bm{x}_i^k}, t_i^{k}) = p
\end{split},
\end{equation}
with $q = 0$ and $p = \infty$, where $\theta(\bm{x},\bm{y})$ is the set of all paths connecting points $\bm{x}$ and $\bm{y}$, and $l$ is the parameterization of a particular path $C$ weighted by function $g_\gamma$.

The total user input effect at point $\bm{x}$ is computed as
\begin{equation}
U_i(\bm{x},t):= -\min_{u_k \in \bm{u}_i,\bm{y}=\bm{x}} u_k(\bm{y},t)+\min_{u_k \in \bm{u}_i^c,\bm{y}=\bm{x}} u_k(\bm{y},t).
\end{equation}

\subsection{Impulsive Control: For Efficient User Interactions}
\label{sec:ImpulseControl}
At a given time \textit{t}, the regions of current and ideal segmentation for a given label \textit{i}, denoted by $\Omega_i(t)$ and $\Omega^*_i(t)$ respectively, can be represented as
\begin{equation}
\begin{split}
\Omega_i(t)&={\Omega^{correct}_i(t)}\cup {\Omega^{mclass}_i(t)}\\
\Omega^*_i(t)&={\Omega^{correct}_i(t)}\cup{\Omega^{nreached}_i(t)},
\end{split}
\end{equation}
where $\Omega^{correct}_i(t)$ is the correct segmentation, $\Omega^{mclass}_i(t)$ is the region mis-classified as the label \textit{i}, and $\Omega^{ureached}_i(t)$ is the region not reached by the label \textit{i}. Note that, $\Omega^{mclass}_i(t)$ is also a unreached region of other labels. See Figure \ref{fig:ErrorDef} for example, where $\Omega^{mclass}_1(t) = \Omega^{ureached}_3(t)$. Therefore, user input can be focused on the unreached regions. Impulsive control may be used to apply the input immediately to the coupled dynamical system, rather than waiting for their effect to reach current fronts as in \cite{KSlice_TMI}.
\begin{figure}[!ht]
\centering
\includegraphics[width = 2.6in]{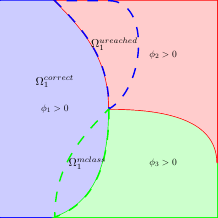}
\caption{An example of segmentation errors   at time \textit{t}. }
\label{fig:ErrorDef}
\end{figure}

Suppose one user input is applied at $\Omega_{\bm{x}_i^k}$ and $t_i^k$. Let $P:\Omega\rightarrow \mathbb{R}$  such that  $H(P(\Omega_{\bm{x}_i^k}))=H(U_i(\Omega_{\bm{x}_i^k}))$. Then, the impulse effect of this input to the coupled system \eqref{equ:dynamicsApprox}, \eqref{equ:dynamicsEstimator} is defined as
\begin{equation}
\label{equ:impulseCoupled}
\begin{split}
\Delta\hat{\phi}^*_i(\Omega_{\bm{x}_i^k}, t^k_i+)&=P(\Omega_{\bm{x}_i^k})-\hat{\phi}^*_i(\Omega_{\bm{x}_i^k}, t^k_i)\\
\Delta\phi_i(\Omega_{\bm{x}_i^k}, t^k_i+) &= \hat{\phi}^*_i(\Omega_{\bm{x}_i^k}, t^k_i+) - \phi_i(\Omega_{\bm{x}_i^k}, t^k_i).
\end{split}
\end{equation}
One simple example of \textit{P} is $P=sign(U_i)$.

\begin{assumption}
Assume that the size of each unreached region $\Omega^{nreached}_i(t)$ can be driven sufficiently small by a \textit{finite} sequence of user input $\{\Omega_{\bm{x}_i^k}(t_i^k)\}$, $k=1,\cdots, K_i$, where $K_i$ is the number of user input for label \textit{i}.
\label{assump:impulse}
\end{assumption}

\begin{nota}
The rationale of this assumption is based on the fact that since $\Omega^{nreached}_i(t)$ is compact, then it has a \textit{finite subcover}, consisting of a finite set $\{\omega_i^k\}$ with $\overline{\omega_i^k}=H(u_i^k(t))$ from user input $u_i^k$  defined in Equation \eqref{equ:uik}. A properly applied $u_i^k$ may increase mis-classified region but it increases correct region as well. Thus, it is reasonable to assume the existence of a ``finite'' number of user inputs required for a desired segmentation.
\end{nota}

\begin{corollary}
Consider the coupled system \eqref{equ:dynamicsApprox}, \eqref{equ:dynamicsEstimator} with the assumption \ref{assump:impulse}. Suppose Theorem \ref{th:CoupledSys} is satisfied. Then, the coupled system with impulse effect \eqref{equ:impulseCoupled} has the same stability property and the final segmentation is sufficiently close to the ideal one.
\end{corollary}
\begin{proof}
Since Theorem  \ref{th:CoupledSys} and assumption \ref{assump:impulse} are satisfied, then $V'(t) \leq 0$ for almost all $t\in \mathbb{R}^+$. In addition, there are only a finite number of impulses. Then, the bound of $V'(t)$ in Equation \eqref{equ:VTCouple} still holds for $t > t_K$, where $t_K$ is the last instant of impulse. What's more, the assumption \ref{assump:impulse} ensures to get an expected segmentation.
\end{proof}

\begin{nota}
Applying user input immediately to the segmentation process has been empirically used in lots of interactive segmentation algorithms \cite{GrabCut, GrowCut, GeodesicSeg_IJCV}. This corollary sheds light on the rationale of this commonly used strategy from the perspective of impulsive control of a feedback system.
\end{nota}

\subsection{Implementation}
Efficiency is a key factor in determining the performance of an interactive system. Thus, it is required to add as less computational load and memory cost as possible to implement the feedback control signal. To this end, inputs to each object is grouped and evolved by using a single unit/array. That is,  the state of $\bm{u}_i(\bm{x},t)$ is recorded by using a single array to reduce memory cost.

In addition, efficient implementations were employed for the original automatic methods. Specifically, for the \textit{region-based active contour model}, a sparse-level set implementation \cite{SparseLS} was used that keeps track of $\phi(x,t)=0$ without re-initialization. In the \textit{distance-based clustering} formulation, since $G$ is locally static (see equation \eqref{equ:GDistTerm}), the evolution reduces to solving a static Hamilton-Jacobi equation \cite{LevelSet_Fedkiw}. Thus, distance information can be computed locally in a monotonic way for which efficient numerical schemes such as those given in \cite{FMM_Tsitsiklis, Dijkstra59} may be applied. If no user input is employed, the algorithm has a similar structure as in image segmentation utilizing the Fast Marching Method as originally proposed in \cite{FMM_Tsitsiklis} and developed and extended in many other works; see \cite{FastMarching_Book, LevelSet_Fedkiw, FMM_Segmentation} and the references therein. Introducing the control framework enables the algorithm to be non-static and as such, results in increased flexibility.

\section{Experimental Results}
\label{sec:Results}
In this section, we first present examples of user effect for both \textit{region-based active contours} and \textit{distance-based clustering}. Then, we demonstrate the advantages (and/or disadvantages) of the proposed control framework by comparing it to a popular interactive segmentation method in terms of \textbf{user's effort} and \textbf{predictability}. Next, we present results of applying the proposed framework for segmenting challenging medical images. Finally, we use examples to illustrate the relations between the control-based method and some existing algorithms.

In this section, the localized region-based active contour energy \cite{Local_Global} was implemented for the \textit{region-based active contour model} and a gradient-based distance measure \cite{MinPath_Cohen} was used for the \textit{distance-based clustering} methodology.

\subsection{Examples of User Effects}
An example of segmenting the dolphin flipper using the proposed region-based active contour method is shown in Figure \ref{fig:RegionACExample}. An illustration of segmenting the dolphin flukes using the proposed distance-based clustering algorithm is given in Figure \ref{fig:DistClusterExample}. As can be seen from these simple examples, the proposed algorithm requires only a small amount of user interactions to correct the segmentations towards ideal boundaries.
\begin{figure}[!ht]
\centering
\subfigure[]{\includegraphics[width = 0.8in]{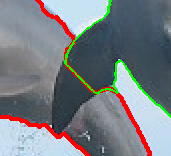}}
\subfigure[]{\includegraphics[width = 0.8in]{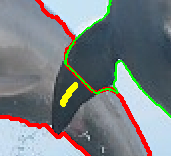}}
\subfigure[]{\includegraphics[width = 0.8in]{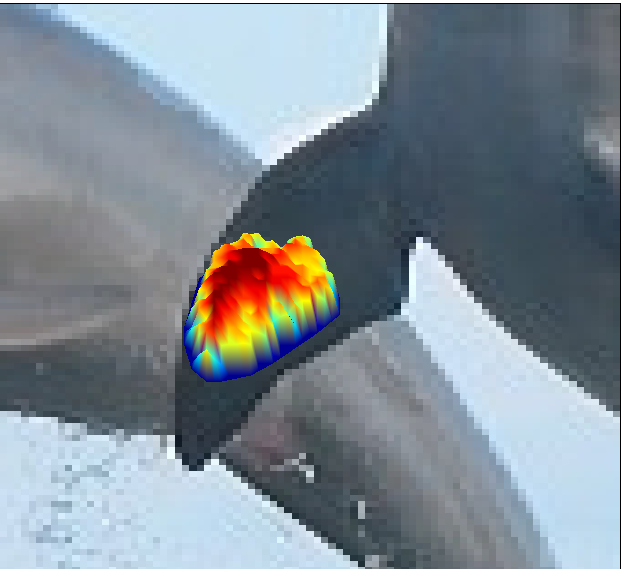}}
\subfigure[]{\includegraphics[width = 0.8in]{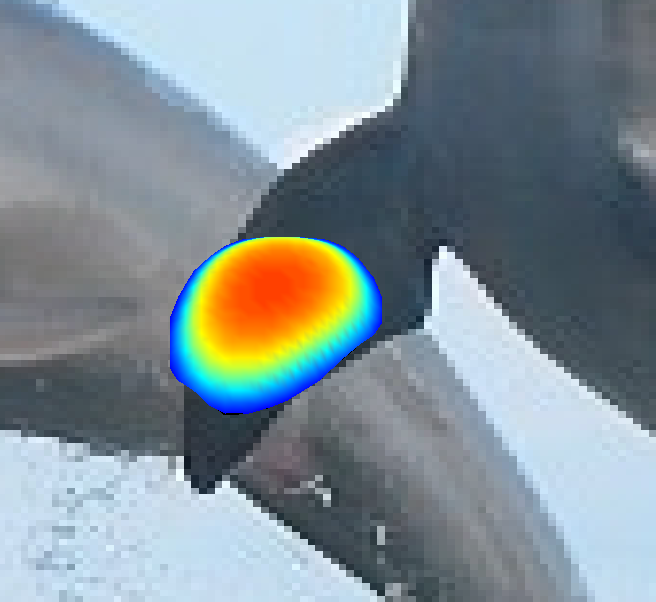}} \\ \vspace{-0.3cm}
\subfigure[]{\includegraphics[width = 0.8in]{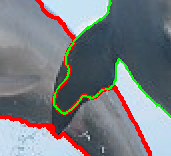}}
\subfigure[]{\includegraphics[width = 0.8in]{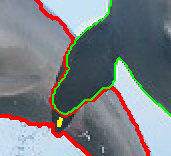}}
\subfigure[]{\includegraphics[width = 0.8in]{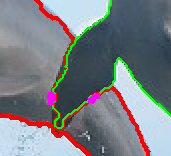}}
\subfigure[]{\includegraphics[width = 0.8in]{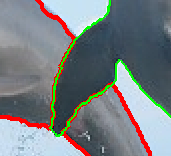}} \vspace{-0.3cm}
\caption{Example of user effect to segment flipper using region-based active contour model. Two dolphins are marked as red (label $L = 1$) and green (label $L = 2$), respectively. (a) Before user interaction. (b) First user input for $L=2$, $\Omega_{\bm{x}_2^{1}}$ at time $t_2^1$ and (c) effect of the user input $u_2^1(\bm{x},t_2^{1+})$ at time $t_2^{1+}$. (d) Propagation of the user input effect $\bm{u}_2(t)$ and (e) the contours at time $t_2$; (f) One more input $\Omega_{\bm{x}_2^{2}}$ for $L=2$; (g) Two inputs for $L=1$ and (h) the final segmentation.}
\label{fig:RegionACExample}
\end{figure}
\begin{figure}[!ht]
\centering
\subfigure[]{\includegraphics[width = 1.0in]{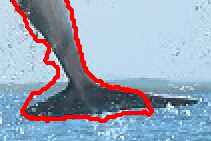}}
\subfigure[]{\includegraphics[width = 1.0in]{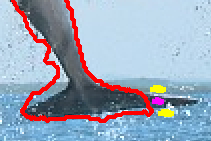}}
\subfigure[]{\includegraphics[width = 1.0in]{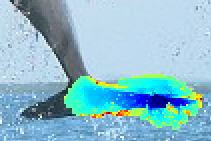}}\\ \vspace{-0.3cm}
\subfigure[]{\includegraphics[width = 1.0in]{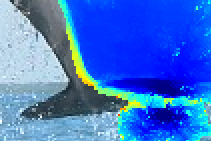}}
\subfigure[]{\includegraphics[width = 1.0in]{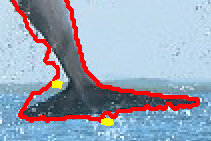}}
\subfigure[]{\includegraphics[width = 1.0in]{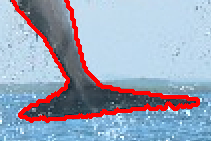}} \vspace{-0.3cm}
\caption{ Example of user effect to segment flukes using distance-based clustering. (a) Before user interaction. (b) Inputs for flukes (red, label $L= 1$) and background (blue, label $L = 3$). (c) Prorogation of $\bm{u}_1(\bm{x},t)$ from $\Omega_{\bm{x}_1}$ (dark blue) and (d) $\bm{u}_3(\bm{x},t)$ from $\Omega_{\bm{x}_3}$ (dark blue); (e) Two inputs for label 3; and (f) the final segmentation.}
\label{fig:DistClusterExample}
\end{figure}
\begin{figure*}[!hbt]
\centering
\subfigure{\includegraphics[height = 0.9in]{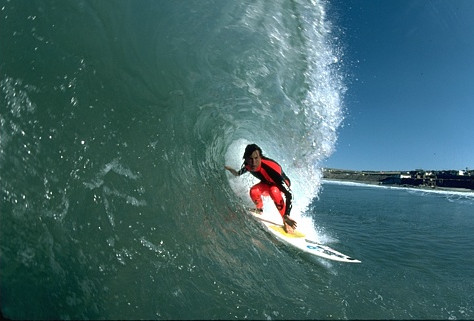}}
\subfigure{\includegraphics[height = 0.9in]{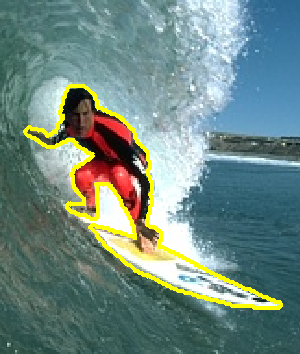}}
\subfigure{\includegraphics[height = 0.9in]{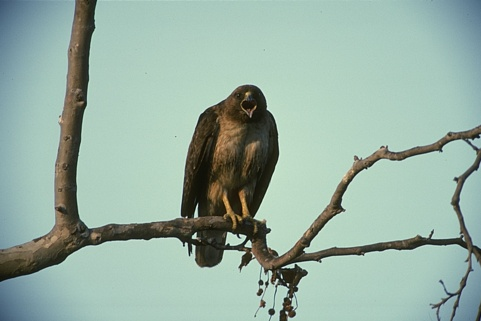}}
\subfigure{\includegraphics[height = 0.9in]{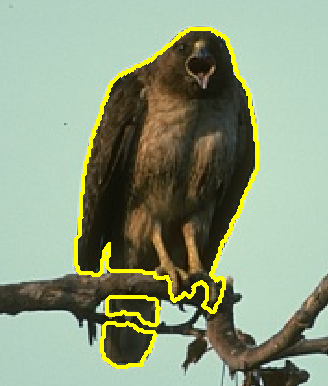}}\\ \vspace{-0.2cm}
\subfigure{\includegraphics[width = 0.9in]{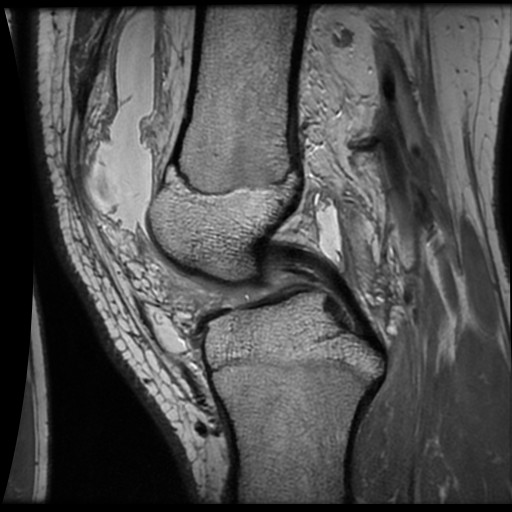}}
\subfigure{\includegraphics[width = 0.9in]{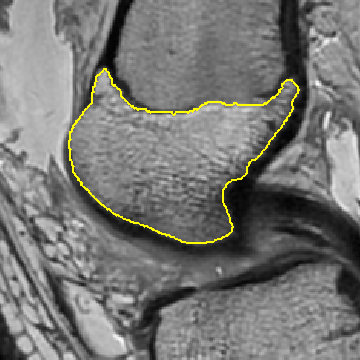}}
\subfigure{\includegraphics[width = 0.9in]{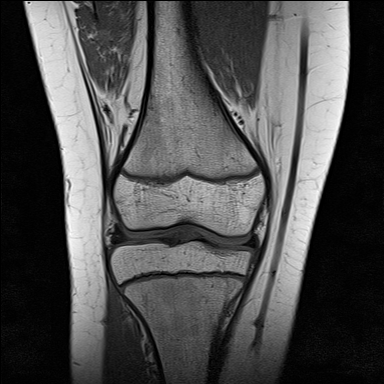}}
\subfigure{\includegraphics[width = 0.9in]{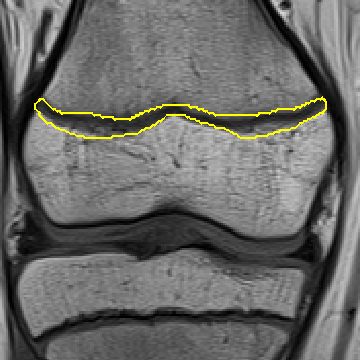}} \vspace{-0.3cm}
\caption{The general (first row) and medical (second row) images used in a quantitative comparison of GrabCut and the proposed algorithm. Manual segmentations are marked in yellow.}
\label{fig:compareManu}
\end{figure*}
\subsection{Effectiveness of The Proposed Control Framework}
\label{sec:Effectiveness}
\subsubsection{Selection of Data}
Two general images from \cite{BerkeleyImageDS} and two medical images were used to quantitatively compare the presented methods with the popular GrabCut algorithm \cite{GrabCut}. The general images considered are given in the first row of Figure \ref{fig:compareManu}.  Specifically, these images where chosen as illustrate strong local contrast and at varying parts of the image (e.g., ambiguous boundaries at the bottom of the bird image). The medical images seen in the second row of Figure \ref{fig:compareManu} present a different issue - the targeted object (epiphysis/physis) has an intensity profile that is comparable to surrounding objects within the background.  In short, the example in Figure \ref{fig:compareManu} is shown to illustrate and motivate the proposed framework.

\subsubsection{Quantitative Comparison of \textbf{User's Effort}}
A location through which the cursor was dragged is defined as an ``actuated voxel''; and the \textbf{total actuated voxels} is a robust indicator of user effort to complete a segmentation.

In this test, the interactive user input via mouse click-and-drag was implemented and measured identically for each algorithm. And the extent around the cursor that mark seed regions in GrabCut were not counted towards the total actuated voxels.

Three experiments were conducted with different initialization for each image. The actuated pixels after initialization are shown in Figure \ref{fig:compareAct}. At termination, all of the segmentations have greater than \%95 overlap with a manually segmented reference.  These results show the different characteristics of these algorithms. The proposed algorithm has a lower mean actuated count in all images and tighter clustering in three of them (except the bird image) across repeated segmentations. The wider cluster in the bird image segmentation from the region-based algorithm is reasonable as it has been observed that region-based segmentation methods seem to be more sensitive in segmenting an object with poor local contrast (preventing the bird's tail and claw bleeding through the tree branches) as compared to the distance-based one.  The differences of performance are significant for medical images, where the background and foreground have very similar intensity distribution, since one iteration of the Grabcut can change the segmentation dramatically. On the contrary, the rapid and continuous visual feedback provided by the provided algorithm prevents the developing of a large error, which reduces user's effort in actuating pixels.
\begin{figure*}[!ht]
\centering
\subfigure[]{\includegraphics[width = 2in]{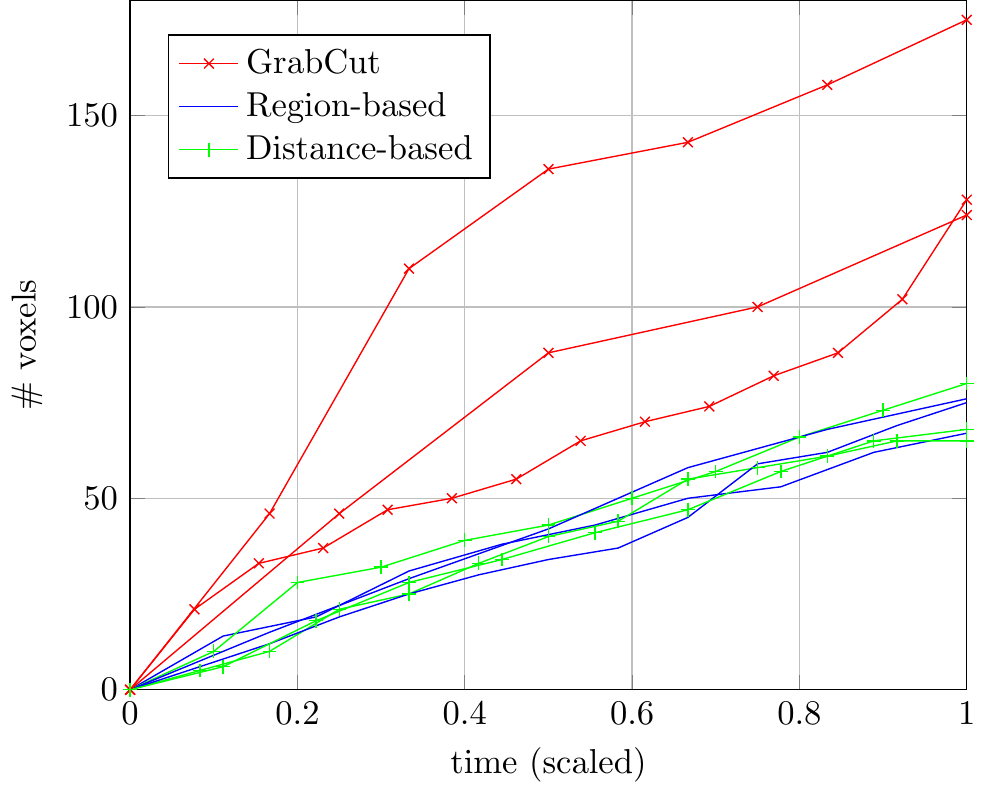}}
\subfigure[]{\includegraphics[width = 2in]{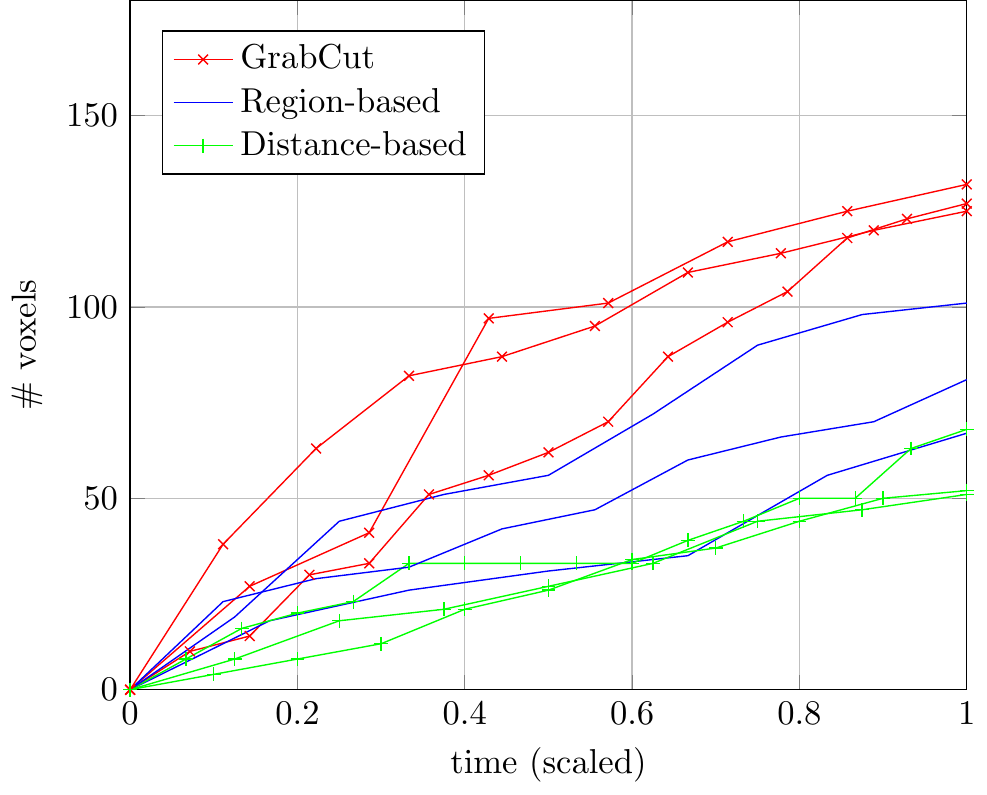}}\\ \vspace{-0.3cm}
\subfigure[]{\includegraphics[width = 2in]{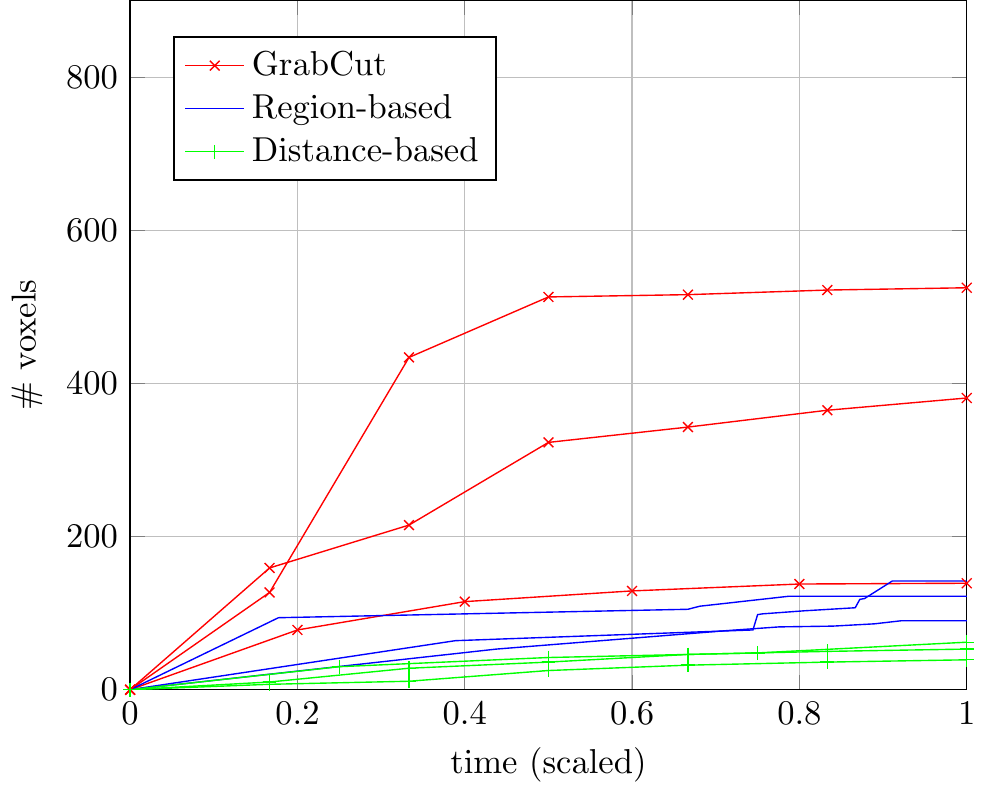}}
\subfigure[]{\includegraphics[width = 2in]{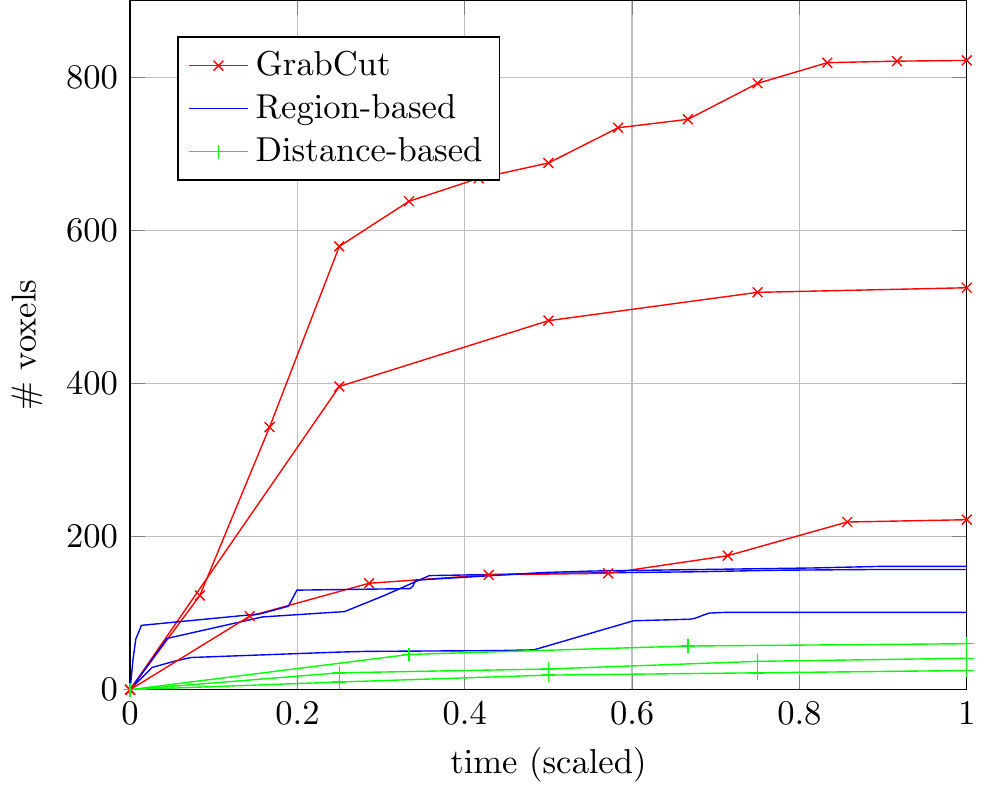}} \vspace{-0.3cm}
\caption{Comparison of actuated voxels over time, after initialization for (a) surfing, (b) bird, (c) epiphysis, and (d) physis images. The proposed algorithm has a lower mean actuated pixels across repeated segmentations.}
\label{fig:compareAct}
\end{figure*}

\subsubsection{Comparison of Algorithmic \textbf{Predictability}}
Predictability of how the segmentation changes in response to mouse strokes is a criterion for practical ease of use. Quantitatively, the change of segmentation is measured by ``reclassified'' voxels, of which the assigned labels change between background and foreground. The \textbf{predictability} is reflected by looking at the dynamic response between user actuated voxels and reclassified voxels recorded over time as $(\# \text{newly actuated voxels}, \# \text{reclassified voxels})$.

Figure~\ref{fig:comparePred} shows the dynamic response from the experiments described in the previous section. Each mark on the figure corresponds to one iteration when new user input was applied. Linear regression lines are overlaid on the data. All algorithms have a very similar dynamic response in the surfing image segmentation, in Figure \ref{fig:predSurf}, since the image has strong local contrast. Differences of predictability become observable in the bird segmentation, where the region-based approach has a tighter distribution along the fitting line because it is less sensitive to the poor-defined boundaries around the bird's tail. The advantages of using control-based algorithm are shown-up in the medical image segmentation. Two issues become apparent for the physis segmentation. First, the distribution of GrabCut data points is quite broad; Second, some of the GrabCut data points are below the dashed pink line, indicating a waste of user effort since there are more voxels actuated than reclassified. The dynamic response of GrabCut makes it hard for a user to predict how much change new mouse strokes will cause.
\begin{figure*}[!ht]
\centering
\subfigure[]{\label{fig:predSurf}\includegraphics[width = 2.0in]{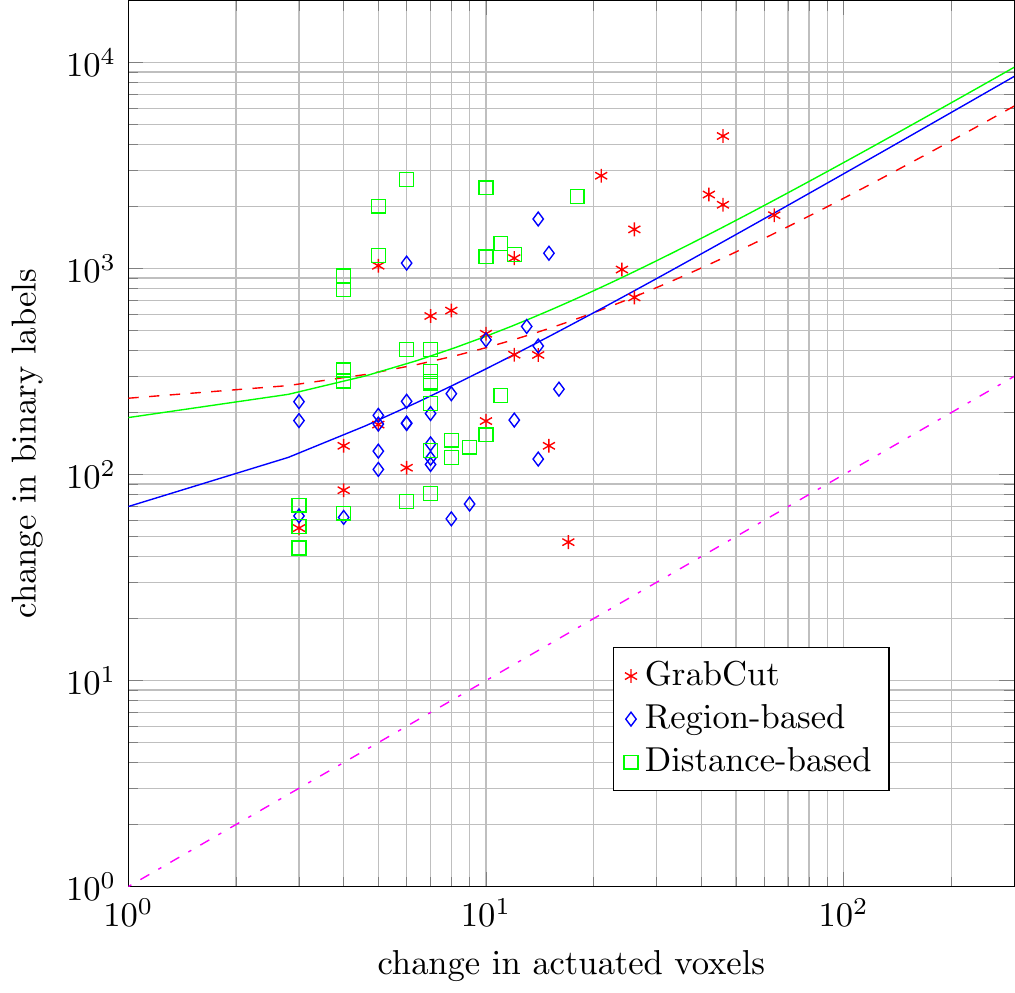}}
\subfigure[]{\label{fig:predBird}\includegraphics[width = 2.0in]{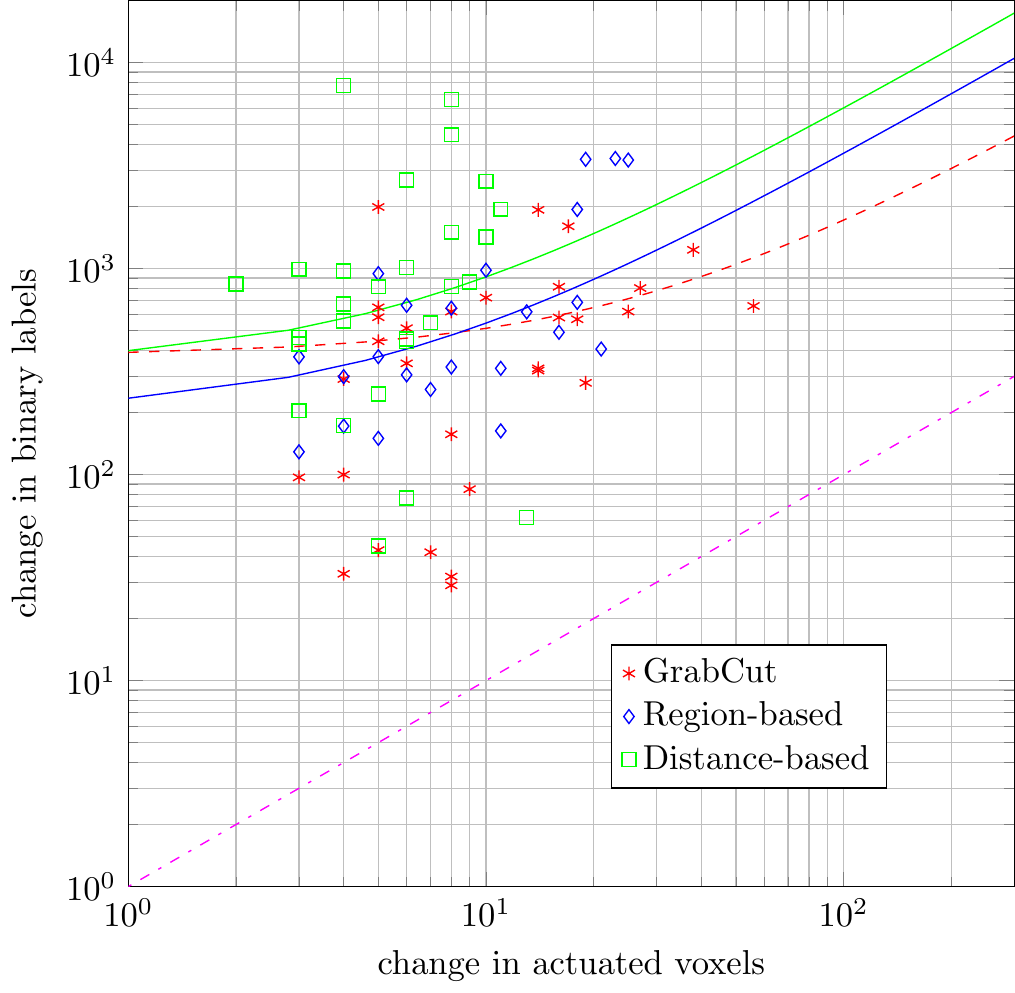}} \\ \vspace{-0.3cm}
\subfigure[]{\includegraphics[width = 2.0in]{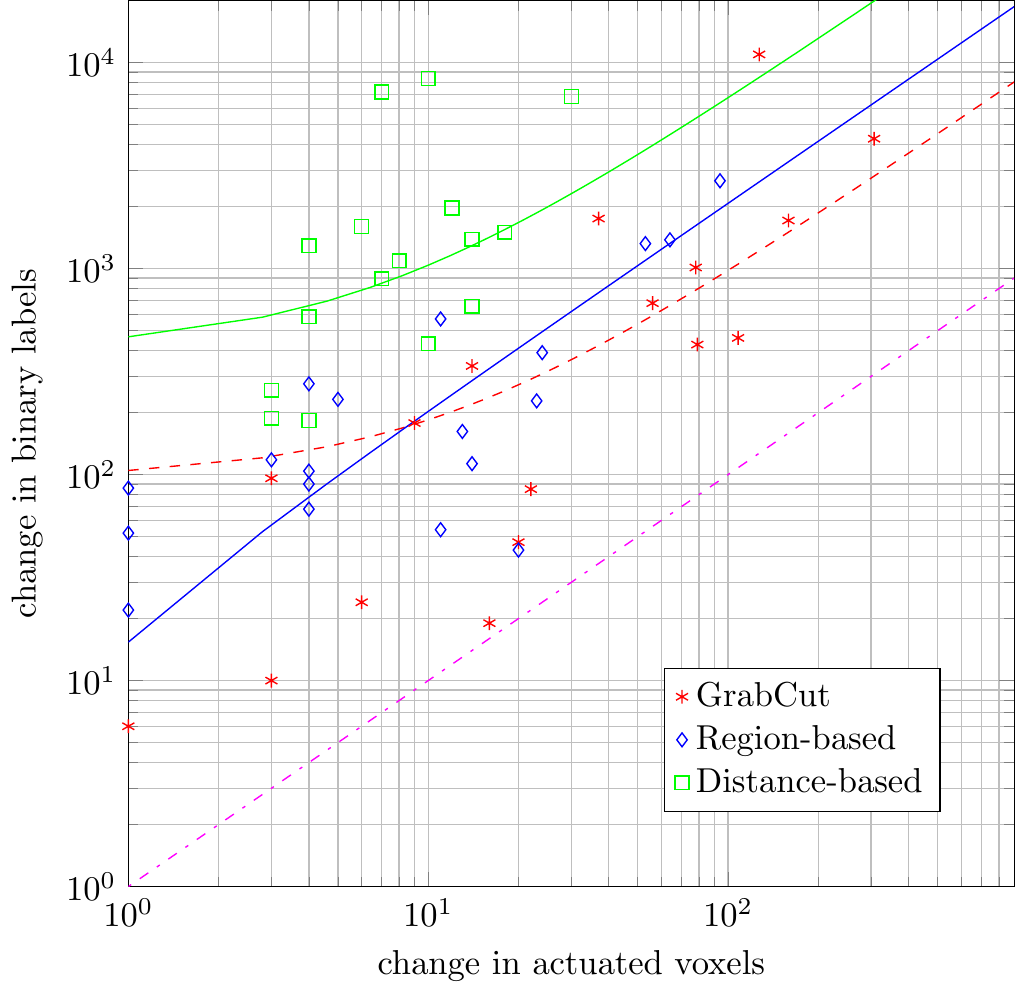}}
\subfigure[]{\includegraphics[width = 2.0in]{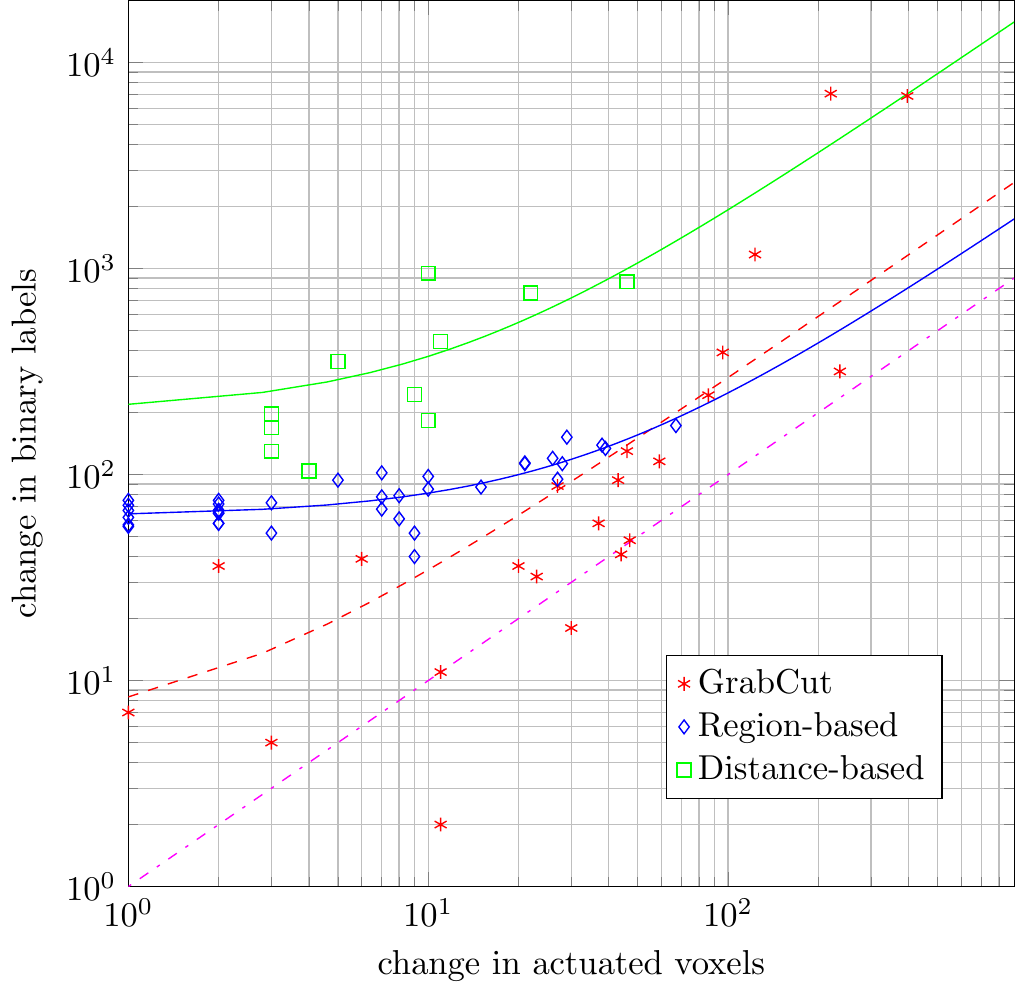}}\vspace{-0.3cm}
\caption{Comparison of dynamic response to user input; data
points and linear fit lines for (a) surfing, (b) bird, (c) epiphysis, and (d) physis images. Points below the dashed pink lines indicate wasted user effort since more additional voxels were actuated than reclassified.}
\label{fig:comparePred}
\end{figure*}

These comparisons do not mean that the proposed algorithm will have better performance than other segmentation algorithms in all cases. It is emphasized that, by forming a closed-loop interactive segmentation framework, these classical algorithms can be revitalized and overcome their disadvantages when used as a single segmentation method in an open loop.

\subsection{Application to Medical Image Segmentation}
The proposed method was tested on real CT volumes. In one experiment, four structures, left/right eye ball, brain stem, and mandible, involving in head-neck radiotherapy contouring were segmented from a CT image. The image size is $512 \times 512 \times 146$ voxels. Due to the high similarity of target structures to surrounding tissues and the required precision for the final segmentation result, the proposed region-based method was chosen in this test based on its robustness to these factors. To show the difference of intra-user performance, three users, who were blind to a reference manual segmentation, were involved in this test. The details of this experiment are  summarized in Table \ref{tab:timingHeadNeck}. The difference is obvious when examining the times required to perform the segmentations. The speedup is roughly $2\times$ according to Table~\ref{tab:timingHeadNeck}. It is also important to notice that the increase in speed becomes more noticeable for large structures that have intricate shapes. The users spent significantly more time outlining the mandible because of its complex boundaries. Additionally, less concentration from the user is required when guided by the interactive method, which further reduces the segmentation time. An example of segmentation is shown in Figure \ref{fig:resultRadThero}. The largest error is at the mandible because of the scanning artifacts. However, the difference in segmentation accuracy is mainly due to the users understanding of these anatomical structures.
\begin{table*}[!hbt]
\footnotesize
\centering
\caption{Quantitative Comparison: Manual v.s. Interactive Approach for Head-Neck Image Segmentation}
\begin{tabular}{|c|c|c|c|c|c|}
\hline
 & \multirow{2}{*}{Manu. Segmentation} & \multirow{2}{*}{Interact. Segmentation} & \multicolumn{3}{c|}{Dice} \\
 \cline{4-6}
 & & &User 1 &User 2 &User 3\\ \hline
Left Eye Ball &3 min 35 sec & 2 min  & 0.85 & 0.88 & 0.84 \\ \hline
Right Eye Ball &3 min 25 sec  & 1 min 30 sec   & 0.87 & 0.94 & 0.87 \\ \hline
Brain Stem &9 min 2 sec & 5 min 30 sec & 0.86 & 0.85 & 0.80  \\ \hline
Mandible &29 min 37 sec & 10 min 15 sec & 0.81 & 0.90 & 0.86  \\
\hline
\end{tabular}
\label{tab:timingHeadNeck}
\end{table*}
\begin{figure*}[!ht]
\centering
\begin{tabular}{c}
\subfigure{\includegraphics[height = 1.1in]{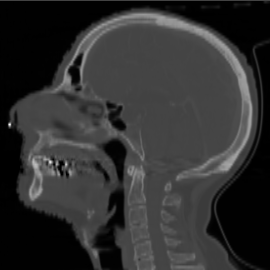}}
\subfigure{\includegraphics[height = 1.1in]{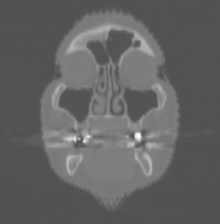}}
\subfigure{\includegraphics[height = 1.1in]{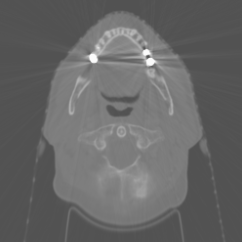}}\\ \vspace{-0.1cm}
\subfigure{\includegraphics[height = 1.1in]{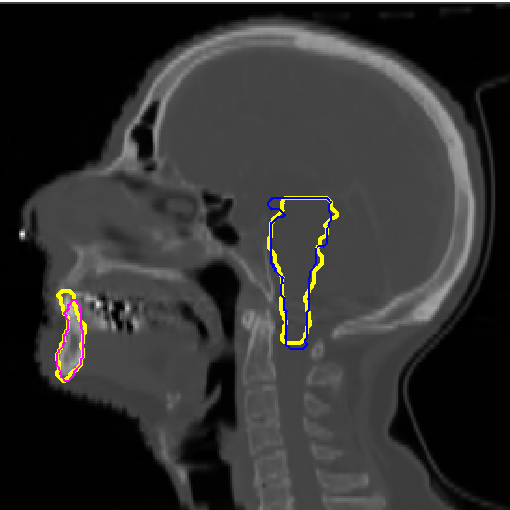}}
\subfigure{\includegraphics[height = 1.1in]{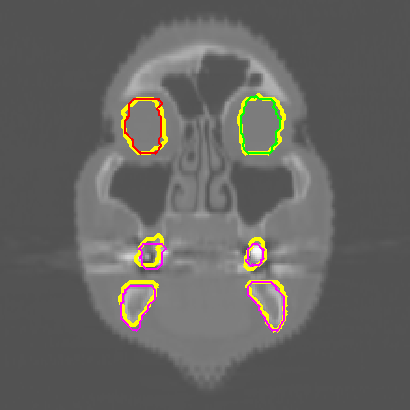}}
\subfigure{\includegraphics[height = 1.1in]{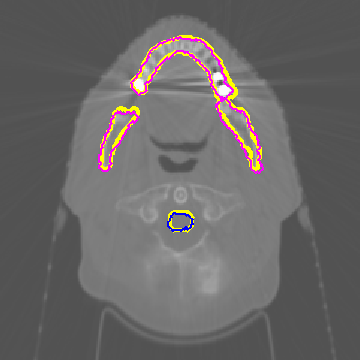}}
\end{tabular}
\begin{tabular}{c}
\subfigure{\includegraphics[height = 1.6in]{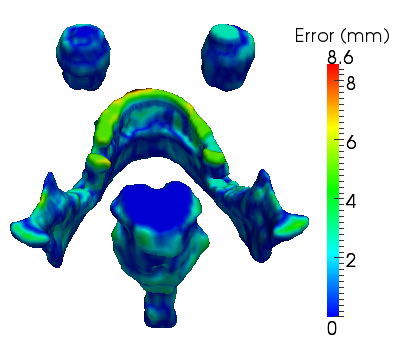}}
\end{tabular}
\caption{Example of using the region-based method to segment the left eye (red), right eye (green), brain stem (blue), and mandible (pink), superimposed over manual segmentations (yellow), in axial,
coronal, and sagittal views, respectively. The distribution of errors is shown on the right.}
\label{fig:resultRadThero}
\end{figure*}

The proposed distance-based method was tested on cardiac chambers segmentation. Three commonly studied chambers in cardiac diseases diagnosis, \textit{i.e.}, left/right ventricles and left atrium, were segmented from three cardiac CT images. The size of image slice along the axial direction is all $512 \times 512$ with the number of slices ranging from 243 to 292. The performance was summarized in Table \ref{tab:timingCardiac}. The proposed method has over $9\times$ speedups on average while maintains close to $0.90$ dice coefficient. The segmentation of the right ventricle is harder than the other two due the lack of contrast along the endocardium and the invisibility of the valves between the right ventricle and atrium.  As can be seen from one example segmentation in Figure \ref{fig:resultCardiac}, large errors are mostly at where chamber boundaries become ambiguous.

\begin{table*}[!ht]
\footnotesize
\centering
\caption{Quantitative Comparison: Manual v.s. Interactive Approach for Cardiac Image Segmentation}
\begin{tabular}{|c|c|c|c|c|c|}
\hline
 & \multirow{2}{*}{Manu. Segmentation} & \multirow{2}{*}{Interact. Segmentation} & \multicolumn{3}{c|}{Dice} \\
 \cline{4-6}
 & & &Left Ventricle &Right Ventricle &Left Atrium \\ \hline
 Case 1 &157 min 46 sec & 18 min 46 sec & 0.90 & 0.90 & 0.89 \\ \hline
 Case 2 &131 min 26 sec & 13 min 45 sec & 0.90 & 0.87 & 0.92 \\ \hline
 Case 3 &88 min 29 sec & 13 min 35 sec & 0.91 & 0.84 & 0.93  \\
\hline
\end{tabular}
\label{tab:timingCardiac}
\end{table*}
\begin{figure*}[!ht]
\centering
\begin{tabular}{c}
\subfigure{\includegraphics[height = 1.0in]{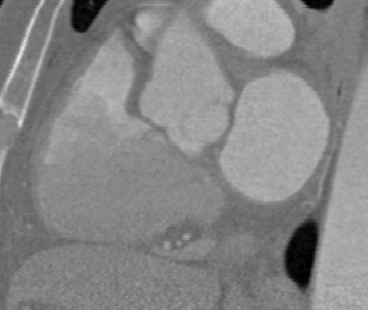}}
\subfigure{\includegraphics[height = 1.0in]{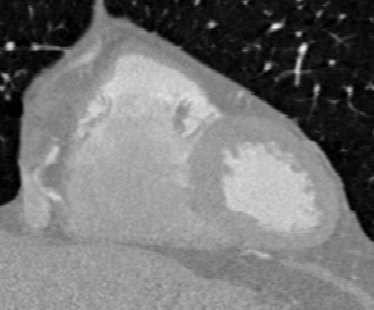}}
\subfigure{\includegraphics[height = 1.0in]{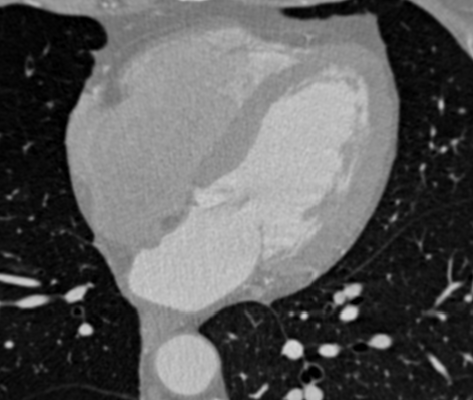}}\\ \vspace{-0.1cm}
\subfigure{\includegraphics[height = 1.0in]{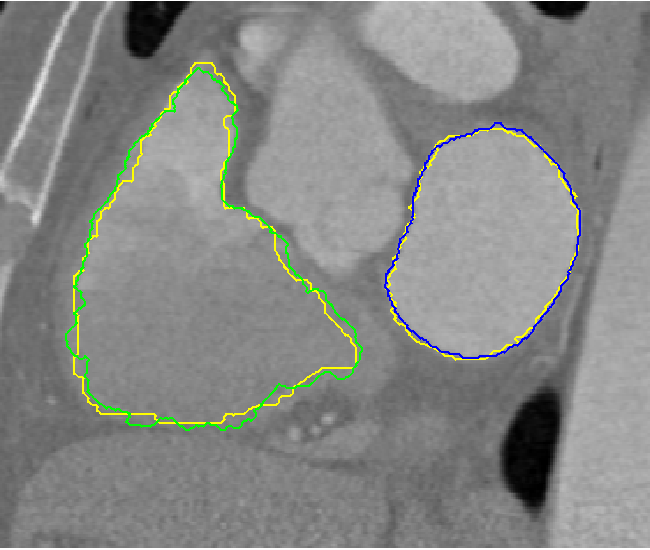}}
\subfigure{\includegraphics[height = 1.0in]{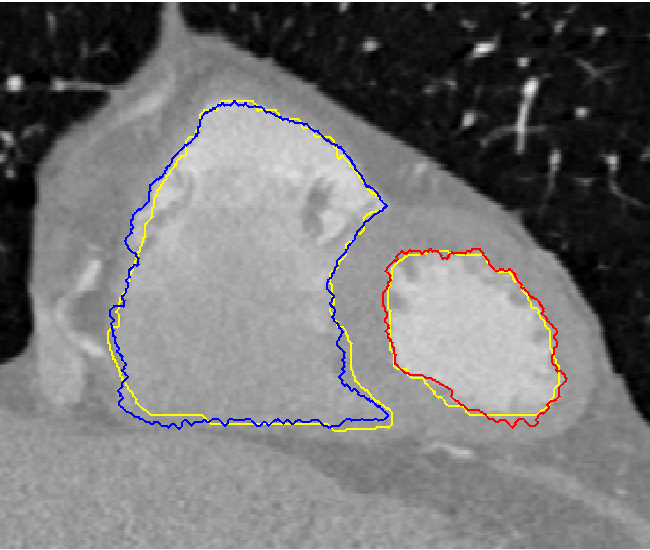}}
\subfigure{\includegraphics[height = 1.0in]{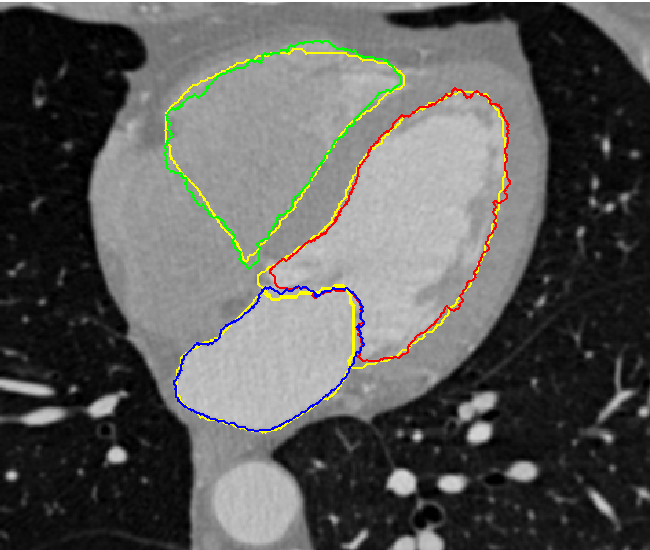}}
\end{tabular}
\begin{tabular}{c}
\subfigure{\includegraphics[height = 1.8in]{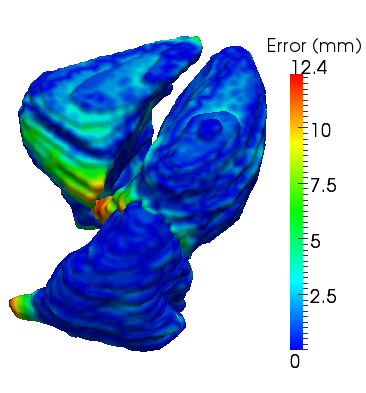}}
\end{tabular}
\caption{Example of using the distance-based method to segment the left ventricle (red), right ventricle (green), and left atrium (blue), superimposed over manual segmentations (yellow), in axial,
coronal, and sagittal views, respectively. The distribution of errors is shown on the right.}
\label{fig:resultCardiac}
\end{figure*}

Note that all the tests were conducted on desktops with regular CPUs. Indeed the proposed method has real-time performance for real 3D medical images, taking into account the factor that the total segmentation time is primarily how long the user takes to evaluate the current segmentation and apply more corrective input.

\subsection{Relation to Existing Interactive Algorithms}
The proposed control framework is reflected implicitly in numerous of existing algorithms. There are some algorithms that directly fit into the proposed framework. For example, the formulations in \cite{Interactive_ProbLevelSet_SPIE, InteractiveLS_ISBI} can be rewritten as
\begin{equation}
\dfrac{\partial \phi}{\partial t} = \delta(\phi)(G(\phi) + \lambda L(\xi)),
\end{equation}
where $L(\cdot)$ is a function modeling user clicks and $\lambda$ is a \textit{global} constant scalar that balances the user's influence to the segmentation. The value of $\lambda$ is usually determined empirically, rather than automatically adjusted to image content as in the proposed framework. As an example, these two methods were applied to segment the epiphysis and physis images with different $\lambda$'s, each with three experiments. As shown in Figure \ref{fig:avgUserEfforts}, many more user inputs are required if $\lambda$ is small, while large $\lambda$ can also increase user's input as it has a similar effect of \textit{excessive input} (see \cite{KSlice_TMI}). Note that a large input was required in segmenting the physis using the method \cite{InteractiveLS_ISBI}, because user input is only applied after an automatic segmentation is finished; thus it can not prevent large errors from occurring. While it is by no means a definitive comparison, the different characteristics of these algorithms were observed by looking at the  minimum average efforts used in segmenting these structures (152, 196, and 120 pixels in epipysis and 169, 404, 140 in physis, respectively, for \cite{Interactive_ProbLevelSet_SPIE,InteractiveLS_ISBI} and the proposed method).
\begin{figure}[!htb]
\begin{center}
	\includegraphics[width = 3in]{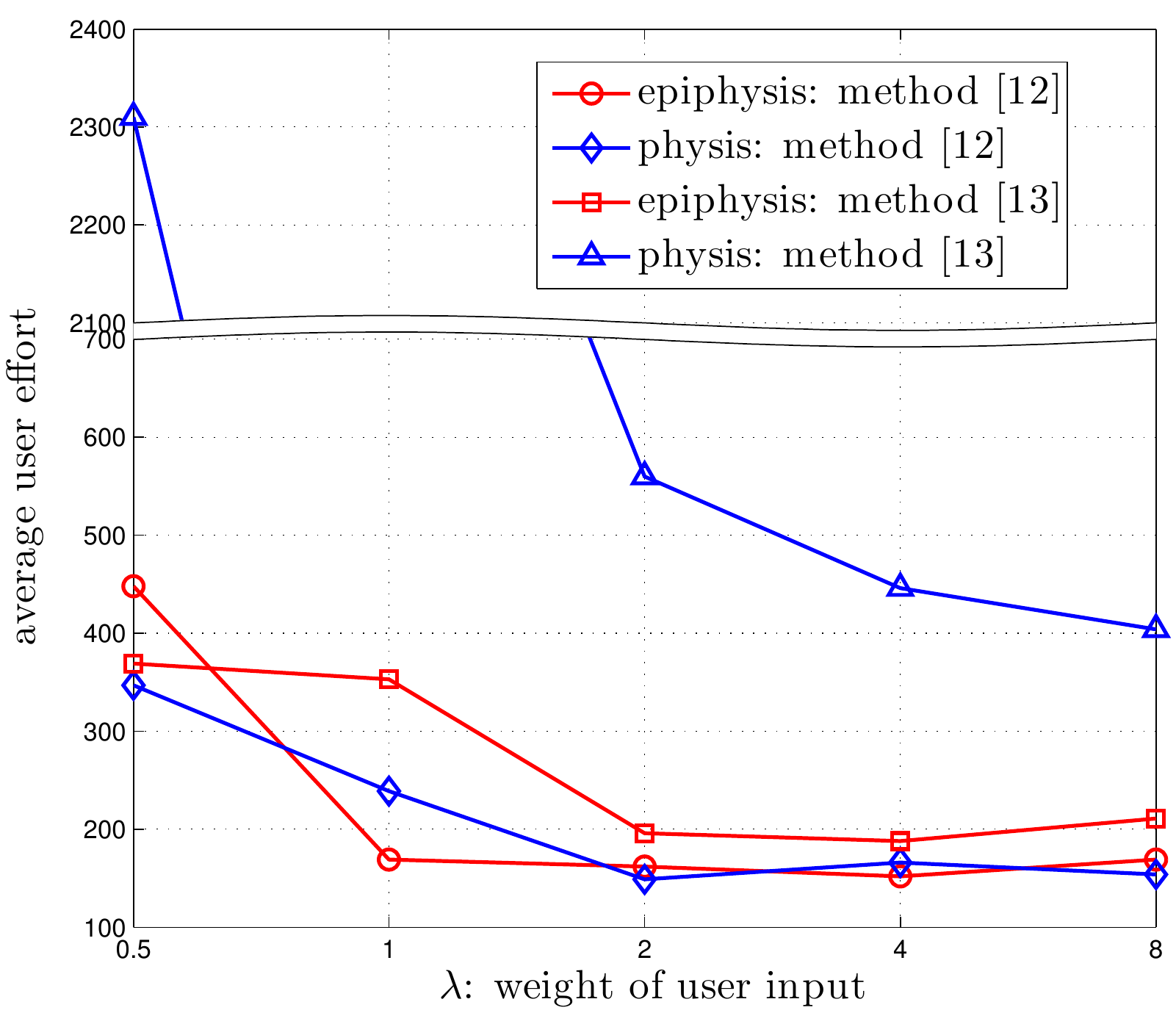}%
\end{center}\vspace{-0.3cm}
\caption{Average user's effort of methods [12] and [13] with varying $\lambda$ for the segmentation of epiphysis and physis.}
\label{fig:avgUserEfforts}
\end{figure}

Another feature of the control-based framework is the system's robustness to ``noisy'' user inputs; see Figure \ref{fig:abruptInput} for example. This is a property inherited from the feedback control design principle that allows admissible input variations \cite{Feedback_Textbook}.
\begin{figure}[!htb]
\centering
\subfigure[]{\label{fig:predSurf}\includegraphics[width = 1.6in]{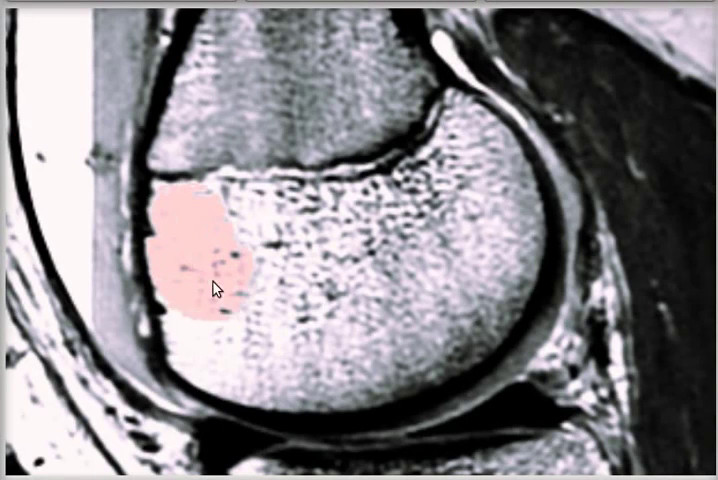}}
\subfigure[]{\label{fig:predSurf}\includegraphics[width = 1.6in]{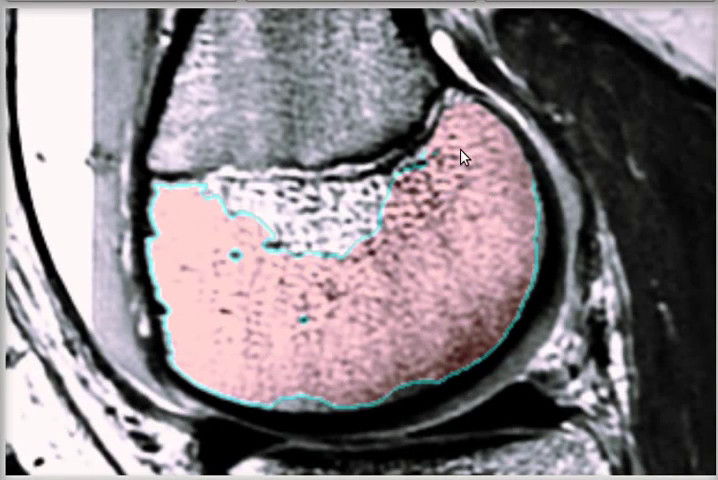}}
\subfigure[]{\label{fig:predSurf}\includegraphics[width = 1.6in]{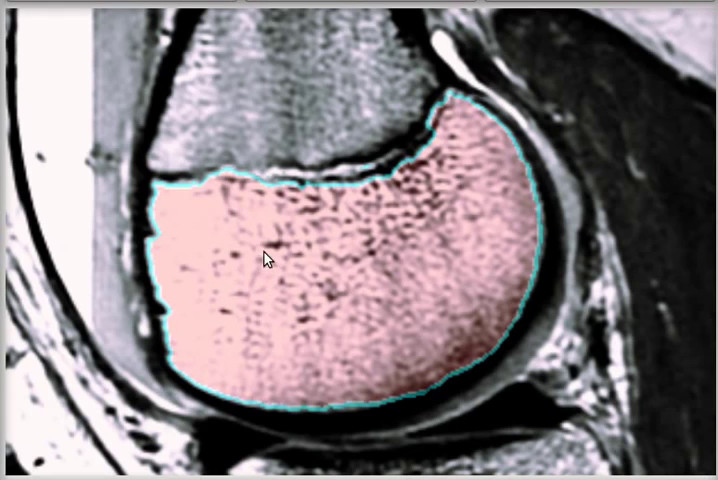}}\\ \vspace{-0.3cm}
\subfigure[]{\label{fig:predSurf}\includegraphics[width = 1.6in]{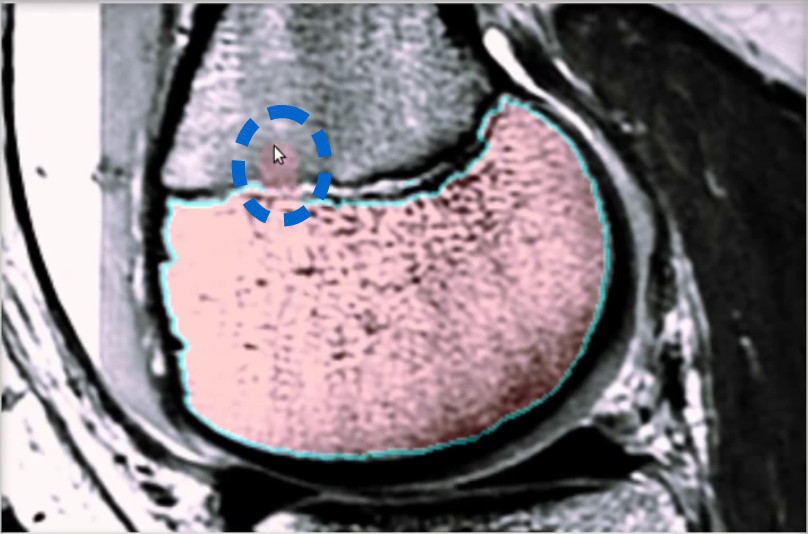}}
\subfigure[]{\label{fig:predSurf}\includegraphics[width = 1.6in]{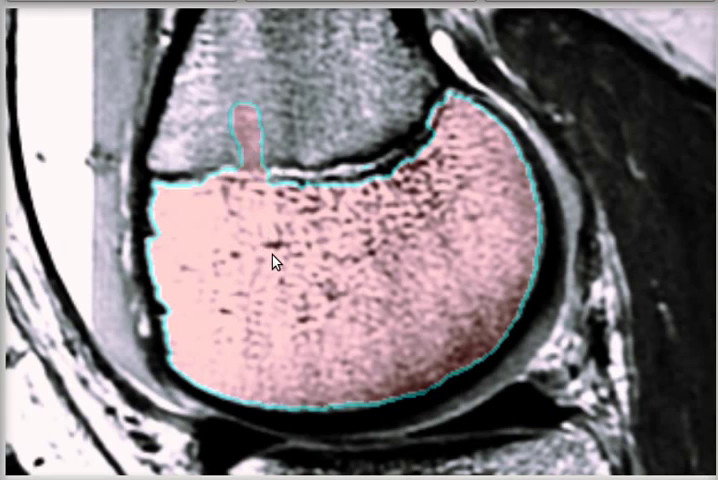}}
\subfigure[]{\label{fig:predSurf}\includegraphics[width = 1.6in]{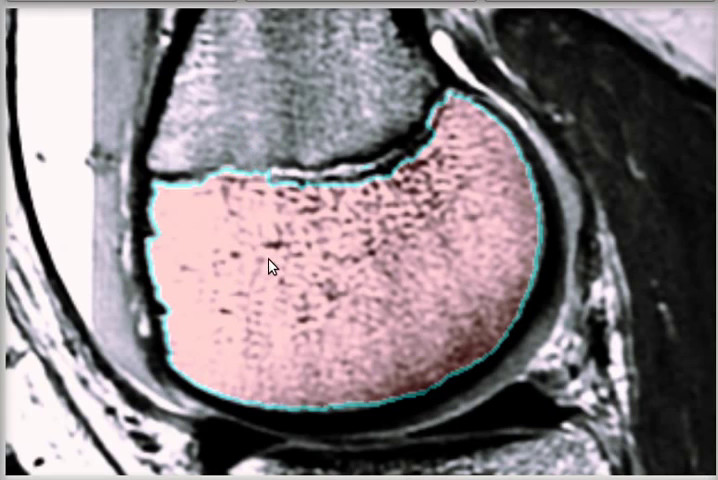}}\vspace{-0.3cm}
\caption{Robustness of the control-based system with respect to ``noisy'' user input. (a)$\sim$ (c) user input for the right segmentation. (d) $\sim$ (e) the system response to ``noisy''  user input (inside blue region). }
\label{fig:abruptInput}
\end{figure}

For some distance-based interactive segmentation algorithms such as \cite{GeodesicSeg_IJCV}, the type of impulsive input described in Equation \eqref{equ:impulseCoupled} was implicitly used to model user interactions.
There are other algorithms that can be formulated into the proposed framework. For instance, the algorithm presented in \cite{GrowCut} can be reformulated as a minimal path problem thereby naturally fitting into the distance-based clustering case.

\section{Conclusion}
\label{sec:Conclusion}
This paper has presented a systematical way of applying control theory to analyze and design an interactive image segmentation algorithm. As an extension of \cite{KSlice_TMI}, the new formulation has wider applications and stronger stability conditions. In particular, the proposed method supports both region- and distance-based metrics, handles multiple-object segmentation, and works for both scalar and vector images. The concept of impulsive control was adopted to model user interactions, which justifies the rationale of widely used empirical strategy from the perspective of feedback control. In addition, conditions for asymptotic and exponential stability are derived, covering different levels of stability requirement.

The experimental results show the effectiveness of adding the proposed control structure to two representative classical methods. Since user's role is seamlessly integrated into a feedback control system, large error can be prevented from developing and user's effort to corrections is guided by the stability condition. Though the examples used in this paper are based on level-sets formulation, the design principle  is generalizable to other interactive segmentation systems that can be described by dynamical systems. It is extensible to discrete systems as well. The focus of this paper is on adopting control theory into image segmentation. There are other factors that determine the performance of an interactive segmentation system. One important topic is informative and efficient visualization to allow effective user interaction \cite{V3D_Vsiaulation}. Simple user scribbles or predefined regions were used in proposed example algorithms. It becomes vital to have more effective user interactions for large scale 3D image volumes.

\section*{Acknowledgment}
This research was supported by the National Center for Research Resources under Grant P41-RR-013218, the National Institute of Biomedical Imaging and Bioengineering under Grant P41-EB-015902 of the National Institutes of Health through the Neuroanalysis Center of Brigham and Women's Hospital, National Institutes of Health 1U24CA18092401A1, and AFOSR grants FA9550-12-1-0319 and FA9550-15-1-0045.

\bibliographystyle{unsrt}
\bibliography{InteractiveSegmentation}
\vspace{-1cm}

\section*{Appendix}

\section*{I: Proof Of Theorem IV.1}
\label{sec:AppdixProof1}
We begin with the following lemma:
\begin{lemma}
\label{lemma:G}
For the image bound defined in Equation \eqref{equ:GM}, there exists at least one label, $i\in \{1,\cdots,N\}$,  such that $G_i(\bm{x},t) < g_M(\bm{x})$ for a point on the zero level set.
\end{lemma}
\begin{proof}
This follows from the (no overlap) condition that $\Omega_i(\bm{x},t)\cap \Omega_j(\bm{x},t)=\emptyset$. Indeed, suppose to the contrary, we have $G_i(\bm{x},t) \geq g_M(\bm{x}), \forall i = 1,\cdots, N$. This only holds when $G_i(\bm{x},t) = g_M(\bm{x}), \forall i = 1,\cdots, N$. Therefore, all the natural dynamics $G_i(\bm{x},t)$ are equal and nonzero at a given point, and so the zero level sets from different regions will ``bleed'' through one another. This contradicts the condition of no overlap between any two regions.
\end{proof}

The proof of the Theorem \ref{theorem:continuous} is as follows.
\begin{proof}
It is straightforward to check that $dV(\xi,t)/dt=0$ if $\xi =0$.

Otherwise, note that $\xi \rightarrow \{-1, 0, 1\}$ as $\epsilon \rightarrow 0$, we first look at the case $\xi \in \{-1, 0, 1\}$. Take $V(\xi, t)$ as the candidate Lyapunov function and differentiate it with respect $t$,
\begin{equation}
\label{equ:theoremC}
\begin{split}
\dfrac{dV(\xi,t)}{dt}
=&\sum_{i=1}^{N}\int_{\Omega}{\xi_i\dfrac{\partial\xi_i}{\partial t}}d\bm{x}\\
=&\sum_{i=1}^{N}\int_{\Omega}{\xi_i\delta(\phi_i)\dfrac{\partial \phi_i}{\partial t}}d\bm{x}\\
=&\sum_{i=1}^{N}\int_{\Omega}{ \delta^2(\phi_i) [G_i \xi_i -\alpha^2_i\xi^2_i]}d\bm{x}\\
\leq&\sum_{i=1}^{N}\int_{\Omega}{ \delta^2(\phi_i) [|G_i| |\xi_i| -\alpha^2_i\xi^2_i]}d\bm{x}\\
=&\sum_{i=1}^{N}\int_{\Omega}{ \delta^2(\phi_i)\xi^2_i [|G_i| -\alpha^2_i]}d\bm{x}\\
<&\sum_{i=1}^{N}\int_{\Omega}{ \delta^2(\phi_i)\xi^2_i [g_M -\alpha^2_i]}d\bm{x}\\
\leq& \quad 0
\end{split},
\end{equation}

The the last inequality holds because of the Lemma~\ref{lemma:G}.

As $\xi_i$ is a \textit{continuous} function of $\epsilon$, then the inequality \eqref{equ:theoremC} holds when $\epsilon$ is sufficiently small.
Therefore, $dV(\xi,t)/dt$ is negative definite under the given condition.

Furthermore, if the condition \eqref{equ:scaleG} is satisfied, by applying the mean value theorem to \eqref{equ:theoremC}, we have
\begin{equation}
\begin{split}
\dfrac{dV(\xi,t)}{dt} \leq&\sum_{i=1}^{N}\int_{\Omega}{ \delta^2(\phi_i)\xi^2_i [|G_i| -\alpha^2_i]}d\bm{x}\\
=&-\dfrac{\tilde{\nu}}{2} \sum_{i=1}^{N}\int_{\Omega}{\delta^2(\phi_i) \xi^2_i}d\bm{x}\\
\leq & -\dfrac{\tilde{\nu}}{\rho} V(\xi, t)
\end{split},
\end{equation}
where $\tilde{\nu} \in \big(0, max_{\{\bm{x},i\}}\big\{2(\alpha^2_i - |G_i|)\big\}$. Note that $V(\xi, t)$ is in fact defined as $\dfrac{1}{2}\Vert\xi\Vert _2^2$. Therefore, the system is exponentially stable with the convergence rate of $\nu=\tilde{\nu}/\rho$ under the given condition.
\end{proof}

\section*{II: Proof Of Theorem IV.2}
\label{sec:AppdixProofTh2}
\begin{proof}
Computing the time derivative $V'(t) = E'(t) + \hat{V}'(t)$,
\begin{equation}
\begin{split}
E'(t) &= \sum_{i=1}^{N}\int_{\Omega}{|U_i|e_{U_i}\dfrac{\partial e_{U_i}}{\partial t}d\bm{x}}\\
&= \sum_{i=1}^{N}\int_{\Omega}{|U_i|e_{U_i}\big(\delta(\hat{\phi}^*_i)\dfrac{\partial \hat{\phi}^*_i}{\partial t}\big)d\bm{x}}
\end{split}
\end{equation}
\begin{equation}
\begin{split}
\hat{V}'(t) &=\sum_{i=1}^{N}\int_{\Omega}{\hat{\xi}_i \dfrac{\partial \hat{\xi}_i}{\partial t} d\bm{x}}\\
&=\sum_{i=1}^{N}\int_{\Omega}{\hat{\xi}_i \big[\delta(\phi_i)\dfrac{\partial\phi_i}{\partial t} -\delta(\hat{\phi}^*_i)\dfrac{\partial\hat{\phi}^*_i}{\partial t}\big]d\bm{x}}.
\end{split}
\end{equation}
Substituting for $\dfrac{\partial\phi_i}{\partial t}$ and $\dfrac{\partial\hat{\phi}^*_i}{\partial t}$,
\begin{equation}
E'(t) =\sum_{i=1}^{N}\int_{\Omega}{\delta^2(\hat{\phi}^*_i)[-|U_i|^2 e^2_{U_i} + |U_i|e_{U_i}\hat{\xi}_i]d\bm{x}}
\label{equ:E'}
\end{equation}
\begin{equation}
\begin{split}
\hat{V}'(t) =&\sum_{i=1}^{N}\int_{\Omega}{\delta^2(\phi_i) [G_i \hat{\xi}_i -\alpha^2_i\hat{\xi}^2_i]}d\bm{x}\\
&-\sum_{i=1}^{N}\int_{\Omega}{\delta^2(\hat{\phi}^*_i) [\hat{\xi}^2_i -|U_i|e_{U_i}\hat{\xi}_i]}d\bm{x}
\end{split}
\label{equ:VHat'}
\end{equation}
Adding Equations \eqref{equ:E'} and \eqref{equ:VHat'} and combining the $\delta^2(\hat{\phi}^*_i)$ terms,
\begin{equation}
\begin{split}
E'(t) + \hat{V}'(t) = &\sum_{i=1}^{N}\int_{\Omega}{\delta^2(\phi_i) [G_i \hat{\xi}_i -\alpha^2_i\hat{\xi}^2_i]}d\bm{x}\\
&-\sum_{i=1}^{N}\int_{\Omega}{\delta^2(\hat{\phi}^*_i) [\hat{\xi}_i -|U_i|e_{U_i}]^2}d\bm{x}
\end{split}
\end{equation}
When $\alpha_i$ satisfies Theorem \ref{theorem:continuous}, it follows that
\begin{equation}
V'(t) \leq -\sum_{i=1}^{N}\int_{\Omega}{\delta^2(\hat{\phi}^*_i)\big[\hat{\xi}_i - |U_i|e_{U_i}\big]^2}d\bm{x}
\end{equation}
\end{proof}

\end{document}